\documentclass[]{pkuai4m}

\usepackage[toc,page,header]{appendix}

\usepackage{minitoc}
\usepackage{multirow}     % 支持 \multirow
\usepackage{multicol}     % 支持多列排版（部分场景）
\usepackage{booktabs}     % \toprule \midrule \bottomrule
\usepackage{colortbl}     % 单元格或行着色
\usepackage[table]{xcolor} % 颜色支持（含 table 选项）
\usepackage{makecell}     % 允许在表格中换行 / 对齐
\usepackage{graphicx}     % 支持 \resizebox
\usepackage{minted}
\usepackage{xcolor}
\usepackage{fontawesome5}
\usepackage{svg}
\usepackage{caption}
\usepackage{fancyvrb}
\usepackage{fvextra}
\usepackage{enumitem}
\usepackage{wrapfig}
\usepackage{mathtools}
\usepackage{needspace}

\usepackage[most]{tcolorbox}
\tcbuselibrary{breakable,skins}

\usepackage{amsthm,amsmath,amssymb}

\allowdisplaybreaks[2]

\usepackage{aliascnt}

\theoremstyle{plain}

\newtheorem{theorem}{Theorem}

\newaliascnt{corollary}{theorem}

\aliascntresetthe{corollary}

\newaliascnt{lemma}{theorem}
\newtheorem{lemma}[lemma]{Lemma}
\aliascntresetthe{lemma}

\newaliascnt{proposition}{theorem}
\newtheorem{proposition}[proposition]{Proposition}
\aliascntresetthe{proposition}

\newaliascnt{cited}{theorem}

\aliascntresetthe{cited}

\newtheorem*{theorem*}{Theorem}
\newtheorem*{problem*}{Problem}
\newtheorem*{remark*}{Remark}
\newtheorem*{claim*}{Claim}
\newtheorem*{conjecture*}{Conjecture}

\theoremstyle{definition}

\newaliascnt{definition}{theorem}

\aliascntresetthe{definition}

\newaliascnt{notation}{theorem}

\aliascntresetthe{notation}

\newaliascnt{remark}{theorem}

\aliascntresetthe{remark}

\newaliascnt{problem}{theorem}

\aliascntresetthe{problem}

\AtBeginDocument{
  \crefname{theorem}{Theorem}{Theorems}
  \Crefname{theorem}{Theorem}{Theorems}

  \crefname{corollary}{Corollary}{Corollaries}
  \Crefname{corollary}{Corollary}{Corollaries}

  \crefname{lemma}{Lemma}{Lemmas}
  \Crefname{lemma}{Lemma}{Lemmas}

  \crefname{proposition}{Proposition}{Propositions}
  \Crefname{proposition}{Proposition}{Propositions}

  \crefname{cited}{Cited Result}{Cited Results}
  \Crefname{cited}{Cited Result}{Cited Results}

  \crefname{definition}{Definition}{Definitions}
  \Crefname{definition}{Definition}{Definitions}

  \crefname{notation}{Notation}{Notations}
  \Crefname{notation}{Notation}{Notations}

  \crefname{remark}{Remark}{Remarks}
  \Crefname{remark}{Remark}{Remarks}

  \crefname{problem}{Problem}{Problems}
  \Crefname{problem}{Problem}{Problems}
}

\tcbset{
  pku theorem box/.style={
    enhanced jigsaw,
    breakable,
    arc=1mm,
    boxrule=0pt,
    fonttitle=\bfseries,
    left=1.5mm,
    right=1.5mm,
    top=1mm,
    bottom=1mm,
    pad at break=1mm,
    before skip=6pt plus 2pt minus 2pt,
    after skip=6pt plus 2pt minus 2pt,
  }
}

\tcolorboxenvironment{theorem}{
    pku theorem box,
    colback=blue!5,
    colframe=blue!5,
}

\tcolorboxenvironment{proposition}{
    pku theorem box,
    colback=blue!5,
    colframe=blue!5,
}

\tcolorboxenvironment{lemma}{
    pku theorem box,
    colback=blue!5,
    colframe=blue!5,
}

\tcolorboxenvironment{corollary}{
    pku theorem box,
    colback=blue!5,
    colframe=blue!5,
}

\tcolorboxenvironment{definition}{
    pku theorem box,
    colback=red!5,
    colframe=red!5,
}

\tcolorboxenvironment{problem*}{
    pku theorem box,
    colback=orange!5,
    colframe=orange!5,
}

\tcolorboxenvironment{conjecture*}{
    pku theorem box,
    colback=orange!5,
    colframe=orange!5,
}

\tcolorboxenvironment{cited}{
    pku theorem box,
    colback=green!5,
    colframe=green!5,
}

\newcommand{\AIVN}{Iteris}

\title{Iteris: Agentic Research Loops for Computational Mathematics}

\author[1,5]{Leheng Chen}
\author[1,5]{Zihao Liu}
\author[1]{Wanyi He}
\author[2,3,4,5\P]{Bin Dong}

\affiliation[]{
$^{1}$School of Mathematical Sciences, Peking University \\
$^{2}$Beijing International Center for Mathematical Research and the New Cornerstone Science Laboratory, Peking University \\
$^{3}$Center for Machine Learning Research, Peking University\\
$^{4}$Center for Intelligent Computing, Great Bay Institute for Advanced Study, Great Bay University \\
$^{5}$Zhongguancun Academy \\
}
\contribution[\P]{Corresponding author}

\def\emailicon{\raisebox{-1.5pt}{\includegraphics[height=1.05em]{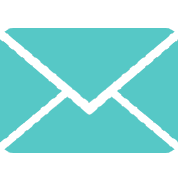}}}

\checkdata[ \emailicon \hspace{0.3em} Correspondence ]{\email{dongbin@math.pku.edu.cn}}

\date{\today}

\abstract{
Recent advances in large language models and agentic AI systems have enabled significant progress in mathematical discovery, from solving competition problems to tackling research-level conjectures. However, open problems in computational mathematics have received comparatively less attention: research in this area often requires not only proofs but also numerical experimentation, adversarial constructions, and algorithm design. In this paper, we introduce  an agentic research system, \AIVN{}, designed for open problems in computational mathematics. We apply \AIVN{} to two open problems from a recent Simons Workshop collection~\cite{AmselEtAl2026OpenQuestions}. In these case studies, \AIVN{} generated numerical evidence, constructions, and proof drafts that led, after expert review and correction, to verified results. The first result is a phase diagram for the asymptotic comparison between conjugate gradient and randomized coordinate descent on power-law spectra; the second is a counterexample showing that QR factorization with column pivoting can fail to select well-conditioned submatrices even under low coherence. These case
studies suggest that agentic AI systems can participate meaningfully in
research workflows for open problems in computational mathematics, while human validation remains essential.

}

\begin{document}
\maketitle

\section{Introduction}
Recent advances in large language models and agentic AI systems have opened new possibilities for mathematical discovery. Frontier reasoning models can now solve increasingly difficult competition and graduate-level problems through natural-language inference~\cite{guo2025deepseek,comanici2025gemini,yang2025qwen3}. Beyond direct problem solving, systems such as FunSearch~\cite{RomeraParedes2024FunSearch} and AlphaEvolve~\cite{NovikovEtAl2025AlphaEvolve} have shown that language models can be embedded in iterative search loops, where candidate programs are generated, evaluated, and improved. More recently, autonomous mathematics systems have begun to target research-level questions, including long-horizon natural-language proof search, conjecture resolution, and human--AI co-mathematical workflows~\cite{FengEtAl2026AutonomousMath,JuEtAl2026ConjectureResolution,ZhengEtAl2026AICoMathematician}. These developments suggest that AI systems are moving from solving isolated mathematical exercises toward participating in genuine mathematical research.

However, open problems in computational mathematics present a different kind of challenge. Progress in this area rarely consists of producing a proof in one pass, nor can it always be reduced to optimizing a single executable objective. A typical research trajectory may involve numerical experimentation, adversarial construction, and proof development. Thus computational mathematics requires an AI system to coordinate several research modes.

In this paper, we introduce \AIVN{}, an agentic research system designed for open problems in computational mathematics. \AIVN{} is organized as an \emph{explore--plan--execute} loop. In each iteration, an exploration agent first probes plausible directions and tests whether the current branch is becoming locally stagnant. A plan agent then reads the global project state and chooses structured tasks for the next iteration. Specialized execution agents carry out the selected tasks, including foundation work, numerical experiment, proof construction, and route review. All project state is maintained through files, which serve both as long-term memory and as structured messages between agents. This design allows \AIVN{} to pursue long research trajectories while keeping research facts separate and checkable.

We apply \AIVN{} to open problems from the recent collection \emph{Linear Systems and Eigenvalue Problems: Open Questions from a Simons Workshop}~\cite{AmselEtAl2026OpenQuestions}. Iteris produced long-horizon exploratory artifacts and proof drafts for these two problems; the final mathematical results were obtained after human verification, mathematical repair where needed, and expository reorganization. The first (\textit{Problem 2.4}) concerns the asymptotic comparison between conjugate gradient (CG) and randomized coordinate descent (RCD) on power-law spectra. We establish a
fixed-parameter phase diagram for the normalized cost ratio between RCD
and CG; see
Theorem~\ref{thm:cg-power-law-phase}. The second (\textit{Problem 4.3}) concerns whether QR factorization with column pivoting (QRCP) reliably selects well-conditioned submatrices from a matrix with orthonormal row. We construct a low-coherence counterexample family showing that QRCP can fail to do so; see Theorem~\ref{thm:qrcp-source-scale-obstruction}. 

\AIVN{} generated the main exploratory and deductive
artifacts that made the two results possible, while the final mathematical
arguments required expert verification and editing. In the CG case, the system found the main phase structure, but
one rate analysis in the \(p<1\) regime initially relied on an unjustified
stronger assumption; this gap was detected by human inspection and repaired
through further human--AI interaction. In the QRCP case, the
system identified the correct obstruction, but the generated
proof trajectory was circuitous and required substantial reorganization into a
readable proof.

Together, these results provide concrete evidence that agentic AI systems can
move beyond benchmark problem solving and contribute to the research workflow
for open problems in computational mathematics.

\paragraph{Contributions.}
\begin{enumerate}[leftmargin=*]
  \item We describe the \AIVN{} system: an explore--plan--execute agentic loop system for distinct research modes in computational mathematics.
  \item We establish a fixed-parameter phase diagram for the CG-versus-RCD
  comparison on power-law spectra, giving asymptotic rates for the normalized
  cost ratio between RCD and CG
  (Theorem~\ref{thm:cg-power-law-phase}).
  \item We answer the QRCP orthonormal-selection question in the negative by
  constructing a low-coherence orthonormal-row counterexample family
  (Theorem~\ref{thm:qrcp-source-scale-obstruction}).
\end{enumerate}

The rest of the paper is organized as follows.  Section~\ref{sec:related-work} reviews related work.  Section~\ref{sec:framework} describes the \AIVN{} system framework.  Section~\ref{sec:open-problems} states the two open problems and their resolutions.  Section~\ref{sec:discussion} discusses broader implications.  Appendices~\ref{app:cg-proof} and~\ref{app:qrcp-proof} give proof details.
\section{Related work}\label{sec:related-work}

\paragraph{AI systems for mathematical discovery.}
AI has increasingly become a tool for mathematical exploration and constructive search. Early work showed that machine learning can reveal patterns in mathematical data and help guide the formulation of new conjectures and theorems~\cite{Davies2021Guiding}. More recent systems use language models as search engines over explicit mathematical or algorithmic objects: FunSearch searches over programs with an automated evaluator~\cite{RomeraParedes2024FunSearch}, while AlphaEvolve extends evaluator-guided search to larger algorithmic code artifacts~\cite{NovikovEtAl2025AlphaEvolve}. Research-level mathematics has also become an explicit target. First Proof proposed research-level mathematical questions as a benchmark for AI reasoning~\cite{AbouzaidEtAl2026FirstProof}; autonomous-mathematics systems such as Aletheia study long-horizon natural-language research processes~\cite{FengEtAl2026AutonomousMath}; and automated conjecture-resolution systems connect informal mathematical reasoning with formal verification~\cite{JuEtAl2026ConjectureResolution}. Another important paradigm is human--AI collaboration: recent AI co-mathematician systems are designed to help pure mathematicians carry out research-level work through ideation, search, computation, and theorem proving~\cite{ZhengEtAl2026AICoMathematician}. These developments mark rapid progress in AI-assisted and increasingly autonomous mathematics, but open problems in applied and computational mathematics remain much less systematically explored.

\paragraph{Agentic research workflows.}
A parallel line of work views research automation as a long-horizon workflow rather than a single model call. ReAct introduced interleaved reasoning and acting in external environments~\cite{YaoEtAl2023ReAct}, and scientific agents such as Coscientist further demonstrate the use of LLM agents for scientific planning and experimentation~\cite{Boiko2024Coscientist}. Closest to full AI-research automation, The AI Scientist frames machine-learning research as an end-to-end agentic process spanning idea generation, literature search, experiment implementation, result analysis, manuscript drafting, and automated review~\cite{LuEtAl2026AIScientist}. These works establish agentic workflows as a promising route to scientific discovery, but they are mostly evaluated in laboratory, or benchmark-mediated settings; computational mathematics requires a tighter integration of numerical experimentation, algorithm design, and proof-sensitive reasoning.
\section{Framework}\label{sec:framework}

Research in computational mathematics moves between several modes: formulating conjectures, running numerical tests, building proofs, and reviewing whether a route should continue. \AIVN{} is an \emph{explore--plan--execute} agentic loop system for coordinating these modes (\Cref{fig:framework}). The system has three main parts: an agentic explore--plan--execute protocol, execution agents for different research modes, and an exploration agent that tests possible directions before the plan agent chooses the next task pool.

\begin{figure}
    \centering
    \includegraphics[width=0.8\linewidth]{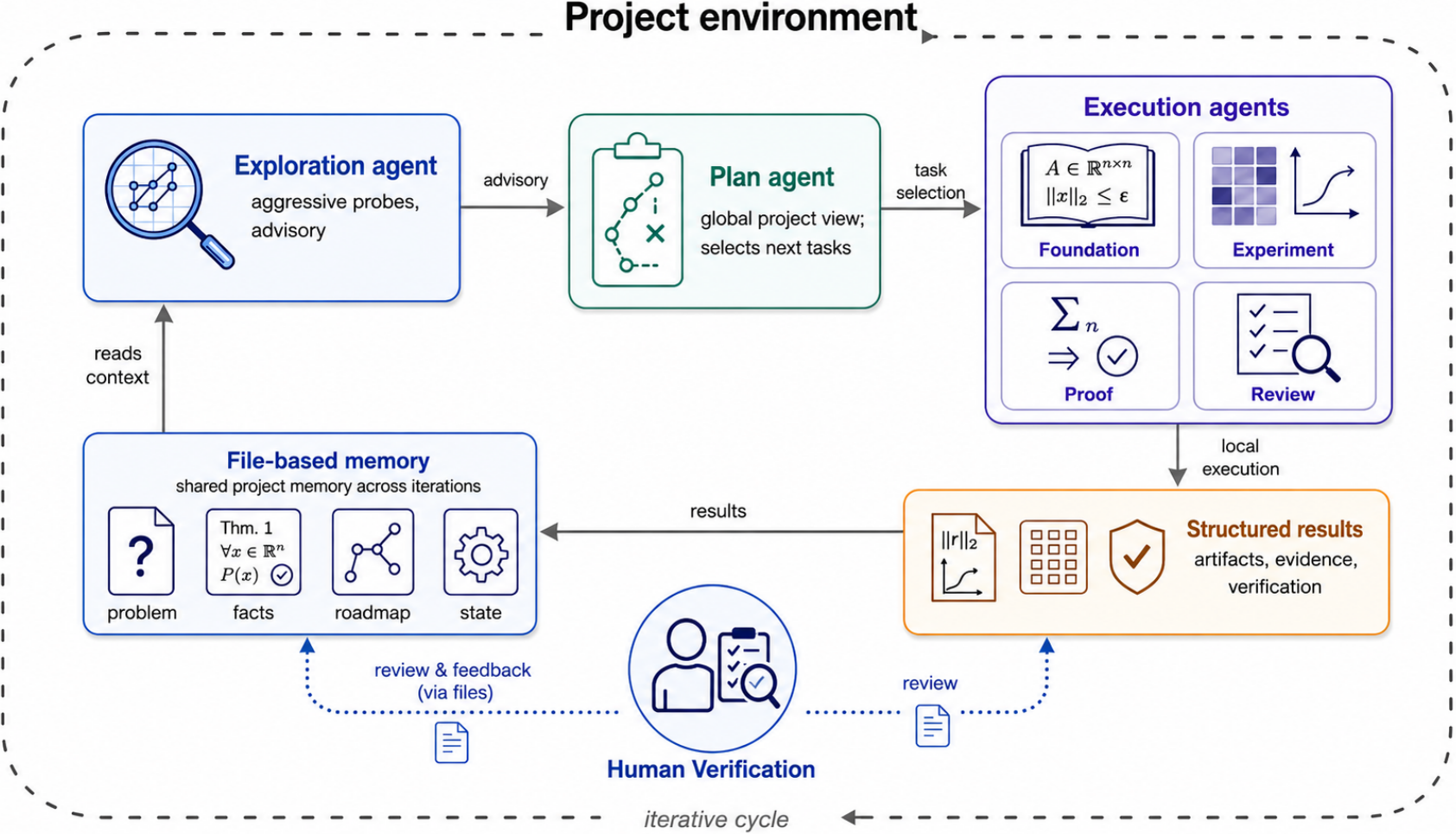}
    \caption{\AIVN{} agent}
    \label{fig:framework}
\end{figure}

\subsection{Agentic explore-plan-execute protocol}\label{subsec:agentic-iteration-protocol}

Each iteration follows the same order: \emph{explore}, \emph{plan}, and \emph{execute}. The three phases separate broad search, global task selection, and local research work:
\begin{enumerate}[leftmargin=*,itemsep=0.35em]
    \item \emph{Explore.} \textbf{Exploration agent} runs probes before the next task pool is fixed. This work consists of light repository-level probing: reading relevant project files, inspecting recent iterations, searching prior facts, and drafting scratch reasoning over a few plausible routes. The exploration agent does not write the official plan. It writes scratch files and an advisory that records the recommended task shape, the local-minimum check, the routes considered, the reason the direction is high value now, and the risks or non-goals that keep the task safe.
    \item \emph{Plan.} \textbf{Plan agent} reads the project files and the exploration advisory. It has the global view of the project: the problem statement, roadmap, project memory, route status, prior facts, frontier state, and recent iteration results. Its job is to decide the task pool for the next iteration. The Plan agent writes \(\texttt{TASK\_POOL.json}\), which specifies each task's category, goal, input files, allowed scope, verification requirements, and result contract.
    \item \emph{Execute.} Each selected \textbf{Execution agent} receives its assigned task from \(\texttt{TASK\_POOL.json}\) as local context. It performs exactly one research action using its task-specific skills and writes structured result files. These result files are added back to project memory and become input for the next iteration.
\end{enumerate}

The protocol is organized around files. Project files store memory across iterations, and they also pass structured information between agents. Thus files play two roles. First, they are memory: they keep the project state, reviewed facts, route decisions, and claim limits stable across many iterations. Second, they are messages: advisories, task pools, and result files move information from one agent to the next.

\subsection{Execution agents for research modes}\label{subsec:executor-agents}

Execution agents implement the local research actions selected by the plan agent. Each execution agent is equipped with a task-specific agent skill: a Markdown-based capability specification that defines its purpose, admissible inputs, execution rules, and output schema. This makes different research modes reusable while keeping each execution step narrow and checkable.

\AIVN{} currently uses four execution agent categories:
\begin{enumerate}[leftmargin=*,itemsep=0.2em]
 \item \textbf{Foundation agent}: audits sources, freezes definitions, drafts conjecture cards, and maintains claim discipline.
 \item \textbf{Experiment agent}: runs numerical work, such as test harnesses, diagnostics, and counterexample searches. Experimental evidence is never treated as proof.
 \item \textbf{Proof agent}: constructs and checks proof artifacts for stable statements, special cases, or route-level lemmas.
 \item \textbf{Review agent}: evaluates evidence, proof gaps, route value, and claim status; it may recommend deepening, reframing, or retiring a branch.
\end{enumerate}

The plan agent selects the execution agent categories needed in each iteration. This is similar to a mathematician deciding whether the next step should be source cleanup, computation, proof work, or strategic review. The selected execution agent does not make global project decisions. It applies the skills specified by \(\texttt{TASK\_POOL.json}\) within the assigned scope and returns structured result files. This separation helps prevent false progress: numerical evidence, proof sketches, reviewed facts, and formal theorem claims are kept distinct.

\subsection{Exploration agent for escaping local inertia}\label{subsec:exploration-agent}

The exploration agent is separated from the plan agent because planning and exploration play different roles. This separation matters as the plan agent does not directly perform research actions. It works from the global project files, so it may keep improving the route that already looks most natural. This can be useful when the route is correct, but it can also lead to local parameter tuning or route inertia. In open problems, the useful next move may require trying a side idea before it is clear that the idea deserves an official task.

The exploration agent fills this gap by running aggressive probes. Its role is not to decide the official task pool, but to produce decision signals that help the plan agent avoid over-committing to the current branch. In practice, this means reviewing recent iterations, checking whether the branch has narrowed into incremental follow-up work, retrieving prior facts when they bear on the decision, and comparing a small number of plausible routes before recommending a candidate task shape. The plan agent may then adopt, modify, or reject this recommendation when writing \(\texttt{TASK\_POOL.json}\).

With this separation, planning remains globally coherent, execution remains locally scoped, and exploration gives the loop a way to search beyond its current path.

Together, these components form the \AIVN{} system. In the open-problems section, we report two success-case trajectories in which \AIVN{} ran for over a hundred iterations per problem and generated artifacts that contributed to the final results; the theorem statements and final proofs were verified, repaired where necessary, and organized by the authors. In our experiments, the generation agent is implemented using the OpenAI Codex agent with GPT-5.5 as the underlying model.
\section{Open problems and resolutions}\label{sec:open-problems}

This section states the two mathematical problems from the Simons workshop collection~\cite{AmselEtAl2026OpenQuestions}. For each problem we give the theorem statement, and a proof sketch. Full proof details appear in Appendices~\ref{app:cg-proof} and~\ref{app:qrcp-proof}. The proof of Problem II has also been verified in Lean\footnote{https://github.com/frenzymath/qrcp-bounded-coherence-obstruction}.

\subsection{Overview}

\begin{table}[h]
\centering
\begin{tabular}{lll}
\toprule
Problem & Resolution type & Details \\
\midrule
CG vs. RCD on power-law spectra & fixed-parameter phase diagram & Appendix~\ref{app:cg-proof} \\
QRCP orthonormal selection & counterexample & Appendix~\ref{app:qrcp-proof} \\
\bottomrule
\end{tabular}
\caption{The two open-problem resolutions reported in this paper.}
\end{table}

\subsection{Problem I: CG versus sketching}

\paragraph{Background.}
Conjugate gradient (CG)~\cite{HestenesStiefel1952,Saad2003} and randomized sketch-and-project methods~\cite{GowerRichtarik2015,LeventhalLewis2010,Nesterov2012} are two classical iterative methods for solving symmetric positive-definite linear systems.  They rely on different mechanisms: CG exploits global spectral structure through Krylov subspaces and residual polynomials, whereas randomized coordinate descent (RCD) and related coordinate sketch-and-project methods process only one coordinate at each step.  The individual steps of RCD are therefore extremely cheap, but it may need more steps; the natural unit of comparison is an epoch, namely \(n\) coordinate updates, rather than a single coordinate step. The Simons workshop collection~\cite{AmselEtAl2026OpenQuestions} \textit{Problem 2.4} asks how these two mechanisms compare asymptotically in a source model with power-law spectra.

\medskip
The problem concerns this epoch-scaled comparison in a regime where
condition number alone is not predictive.  The stopping times depend jointly
on spectral decay, the source distribution of the right-hand side, and the
orientation of the eigenvectors relative to the coordinate basis. Related
work connects sketch-and-project convergence to randomized SVD~\cite{DerezinskiRebrova2024}
and analyzes subspace-constrained RCD~\cite{LokRebrova2025}.  We ask, for
each fixed \((p,\epsilon)\), which method has the asymptotic advantage and
whether the answer forms a phase diagram.

\medskip
Fix \(p>0\) and let
\[
A_n
=
U_n\operatorname{diag}(1^{-p},2^{-p},\ldots,n^{-p})U_n^*,
\qquad
b_n=A_nz_n,
\]
where \(U_n\in\mathbb C^{n\times n}\) is Haar unitary and \(z_n\) is an
independent uniformly distributed complex unit vector.  Thus the exact solution
is
\[
x_{\star,n}=A_n^{-1}b_n=z_n.
\]
Both algorithms start from \(x_0=0\), and errors are measured in the
\(A_n\)-norm
\[
\|v\|_{A_n}^2=v^*A_nv.
\]

The CG method is exact, unpreconditioned conjugate gradient.  The randomized
method is the coordinate sketch-and-project method used in the open problem:
at step \(t\), it samples \(i_t\in\{1,\ldots,n\}\) with probability
\[
\mathbb P(i_t=i\mid A_n)=\frac{A_n(i,i)}{\operatorname{tr}(A_n)}
\]
and performs the exact coordinate update
\[
x_{t+1}
=
x_t+
\frac{(b_n-A_nx_t)_{i_t}}{A_n(i_t,i_t)}e_{i_t}.
\]
We write \(I=(i_0,i_1,\ldots)\) for the coordinate-sampling sequence.

For a fixed relative squared-error threshold \(\epsilon\in(0,1)\), define
\[
T_{\mathrm{CG},n}(\epsilon,p)
=
\min\left\{
t\ge0:
\|x_t^{\mathrm{CG}}-x_{\star,n}\|_{A_n}^2
\le
\epsilon\|x_0-x_{\star,n}\|_{A_n}^2
\right\},
\]
and
\[
T_{\mathrm{RCD},n}(\epsilon,p;I)
=
\min\left\{
t\ge0:
\|x_t^{\mathrm{RCD}}(I)-x_{\star,n}\|_{A_n}^2
\le
\epsilon\|x_0-x_{\star,n}\|_{A_n}^2
\right\}.
\]

Since one RCD epoch consists of \(n\) coordinate updates, we compare CG steps
with RCD epochs through
\[
\rho_n(\epsilon,p)
=
\frac{T_{\mathrm{RCD},n}(\epsilon,p;I)/n}
     {T_{\mathrm{CG},n}(\epsilon,p)}.
\]
The goal is to determine the fixed-parameter asymptotic behavior of
\(\rho_n(\epsilon,p)\) as \(n\to\infty\).

Before stating the theorem, we introduce the CG-side quantity that will govern
the phase diagram.  It measures what fixed-degree residual polynomials can
achieve in the power-law limit. For
\(0<p<1\), the limiting tail measure has only finitely many finite moments.
Since a degree-\(k\) residual polynomial involves moments up to order \(2k\),
 its limiting quadratic form is finite precisely when $(2k+1)p<1$.
 
Define
\[
K(p)=\max\{k\ge0:(2k+1)p<1\},
\]
and the limiting moment measure
\[
\nu_p(dy)=\frac{1-p}{p}y^{-1/p}\,dy,
\qquad y\in[1,\infty).
\]
We define the limiting CG floor by
\[
F(p)
=
\min_{\deg r\le K(p),\ r(0)=1}
\int r(y)^2\,\nu_p(dy).
\]
This is a finite-dimensional moment problem.  Equivalently, if
\[
m_s(p)=\int y^s\,\nu_p(dy)=\frac{1-p}{1-(s+1)p},
\qquad 0\le s\le 2K(p),
\]
and
\[
G_{K(p)}(p)=\bigl(m_{a+b}(p)\bigr)_{0\le a,b\le K(p)},
\]
then
\[
F(p)=\frac{1}{\bigl(G_{K(p)}(p)^{-1}\bigr)_{00}}.
\]
Thus \(F(p)\) is the constrained minimum of the finite Hankel moment matrix
associated with \(\nu_p\).

Equivalently, if \(G_K(p)_{ab}=\int y^{a+b}\,d\nu_p(y)\) on the active range,
then
\[
F(p)=\frac{1}{(G_{K(p)}(p)^{-1})_{00}}.
\]
Intuitively, \(F(p)\) is the lowest relative error level that fixed-degree CG
can stably reach in the limiting active problem.  The position of
\(\epsilon\) relative to \(F(p)\) is the primary divider in the
fixed-parameter phase diagram.

\newpage

\begin{theorem}[Fixed-parameter CG/RCD phase diagram with rate upper bounds]
\label{thm:cg-power-law-phase}
Fix \(p>0\) and \(\epsilon\in(0,1)\).  Then the zero-ratio regimes are:
\[
\begin{array}{c|c|c}
\text{Regime}
&
\text{Conclusion}
&
\text{Proved upper bound}
\\
\hline
p>1
&
\rho_n(\epsilon,p)\xrightarrow{\mathbb P}0
&
\rho_n(\epsilon,p)=O_{\mathbb P}(n^{-1})
\\[1mm]
p=1
&
\rho_n(\epsilon,1)\xrightarrow{\mathbb P}0
&
\rho_n(\epsilon,1)
=
O_{\mathbb P}\!\left(\dfrac{\log n}{a_n}\right)
\\[2mm]
1/3<p<1
&
F(p)=1,\quad \rho_n(\epsilon,p)\xrightarrow{\mathbb P}0
&
\rho_n(\epsilon,p)=O_{\mathbb P}(1/b_n)
\\[1mm]
p=1/3
&
F(p)=1,\quad \rho_n(\epsilon,1/3)\xrightarrow{\mathbb P}0
&
\rho_n(\epsilon,1/3)=O_{\mathbb P}(1/b_n)
\\[1mm]
0<p<1/3,\ \epsilon<F(p)
&
\rho_n(\epsilon,p)\xrightarrow{\mathbb P}0
&
\rho_n(\epsilon,p)=O_{\mathbb P}(1/b_n).
\end{array}
\]
Here, in the \(p=1\) row the bound holds for every deterministic
\(a_n\to\infty\) satisfying
\[
\log a_n=o(\log n),
\]
whereas in the \(O_{\mathbb P}(1/b_n)\) rows it holds for every deterministic
\(b_n\to\infty\) satisfying
\[
b_n\log(1+b_n)=o(\log n).
\]

The remaining regimes are as follows.
\begin{enumerate}[label=(\arabic*)]
\item If \(0<p<1/3\) and \(F(p)<\epsilon<1\), then there exists
\(\eta(p,\epsilon)>0\) such that
\[
\mathbb P\!\left(\rho_n(\epsilon,p)>\eta(p,\epsilon)\right)\to1.
\]

\item If \(0<p<1/3\) and \(\epsilon=F(p)\), then there exist
\(\eta(p)>0\) and \(c(p)>0\) such that
\[
\liminf_{n\to\infty}
\mathbb P\!\left(\rho_n(F(p),p)>\eta(p)\right)\ge c(p).
\]

\item Moreover, on the endpoint/lower-subband critical regimes
\[
p=\frac1{2K+3}
\quad\text{or}\quad
\frac1{2K+3}<p\le\frac1{2K+2},
\qquad K\ge1,
\]
the critical obstruction strengthens to
\[
\mathbb P\!\left(\rho_n(F(p),p)>\eta(p)\right)\to1
\]
for some \(\eta(p)>0\).
\end{enumerate}
\end{theorem}

This theorem is pointwise in the fixed pair \((p,\epsilon)\).  The rate bounds
are conservative stochastic upper bounds, included only to quantify the
zero-ratio regimes.  We do not claim that these bounds are sharp.

\subsubsection{Proof sketch}

First, we reduce exact CG to a finite-dimensional residual-polynomial
minimum.  Write
\[
A_n=U_n\Lambda_nU_n^*,
\qquad
\Lambda_n=\operatorname{diag}(1^{-p},2^{-p},\ldots,n^{-p}),
\qquad
\alpha=U_n^*z_n.
\]
For \(d\ge0\), define
\[
M_{d,n}(p)
=
\min_{\deg q\le d,\ q(0)=1}
\frac{\sum_{j=1}^n \lambda_j |q(\lambda_j)|^2|\alpha_j|^2}
     {\sum_{j=1}^n \lambda_j |\alpha_j|^2}.
\]
The standard CG optimality property gives
\[
T_{\mathrm{CG},n}(\epsilon,p)
=
\min\{d\ge0:M_{d,n}(p)\le\epsilon\}.
\]
Since \(z_n\) is uniform on the complex unit sphere, the spectral weights have
the representation
\[
(|\alpha_1|^2,\ldots,|\alpha_n|^2)
\stackrel d=
\frac{(E_1,\ldots,E_n)}{\sum_{i=1}^nE_i},
\]
where \(E_1,\ldots,E_n\) are independent mean-one exponential random
variables.  Thus CG stopping becomes a random weighted moment problem.

Second, for \(0<p<1\), we identify the fixed-degree CG limit.  Rescale the
spectrum by
\[
Y_{j,n}=\left(\frac nj\right)^p,
\qquad
S_{a,n}=\sum_{j=1}^nY_{j,n}^aE_j,
\]
and define the random probability measure
\[
\mu_n
=
\frac1{S_{1,n}}
\sum_{j=1}^nY_{j,n}E_j\,\delta_{Y_{j,n}}.
\]
Then
\[
M_{d,n}(p)
=
\min_{\deg r\le d,\ r(0)=1}
\int r(y)^2\,\mu_n(dy).
\]
For fixed degrees, the relevant moments of \(\mu_n\) converge to those of
\[
\nu_p(dy)=\frac{1-p}{p}y^{-1/p}\,dy,
\qquad y\in[1,\infty).
\]
A degree-\(k\) residual polynomial requires moments up to order \(2k\), and
these moments are finite exactly when
\[
(2k+1)p<1.
\]
This gives
\[
K(p)=\max\{k\ge0:(2k+1)p<1\},
\]
and the limiting CG floor
\[
F(p)=
\min_{\deg r\le K(p),\,r(0)=1}
\int r(y)^2\,\nu_p(dy).
\]
The fixed-degree limit is
\[
M_{d,n}(p)\xrightarrow{\mathbb P}F_{\min(d,K(p))}(p).
\]

Third, this floor gives the main \(0<p<1\) phase mechanism.  If
\(\epsilon<F(p)\), then every fixed CG degree remains above the threshold
with high probability, so
\[
T_{\mathrm{CG},n}(\epsilon,p)\to\infty
\quad\text{in probability}.
\]
At the same time, source RCD satisfies the epoch upper bound
\[
T_{\mathrm{RCD},n}(\epsilon,p;I)/n=O_{\mathbb P}(1).
\]
Therefore
\[
\rho_n(\epsilon,p)\to0.
\]
This includes the whole row \(1/3\le p<1\), because then \(K(p)=0\) and
\(F(p)=1\).

If \(0<p<1/3\) and \(F(p)<\epsilon<1\), the fixed-degree limit gives bounded
CG stopping:
\[
T_{\mathrm{CG},n}(\epsilon,p)=O_{\mathbb P}(1).
\]
On the other hand, a small-support argument shows that RCD still needs a
positive fraction of an epoch.  Hence the ratio is bounded away from zero with
high probability.

Fourth, the remaining qualitative cases are handled separately.  When
\(p>1\), the spectrum is summable:
\[
\sum_{j=1}^{\infty}j^{-p}<\infty.
\]
The RCD process has a Hilbert-space random-projection limit, and a fixed
number of RCD updates eventually reduces the source error below the threshold.
Thus
\[
T_{\mathrm{RCD},n}(\epsilon,p;I)=O_{\mathbb P}(1),
\]
and since \(T_{\mathrm{CG},n}(\epsilon,p)\ge1\),
\[
\rho_n(\epsilon,p)=O_{\mathbb P}(n^{-1}).
\]
At the logarithmic boundary \(p=1\), RCD satisfies
\[
T_{\mathrm{RCD},n}(\epsilon,1;I)/n=O_{\mathbb P}(\log n),
\]
while the CG polynomial minimum remains above any fixed \(\epsilon<1\) for
super-logarithmically growing degrees.  This gives
\[
\rho_n(\epsilon,1)\to0.
\]

The critical boundary \(\epsilon=F(p)\) is decided by the first inactive
Schur-complement correction.  In every nonvacuous active band, this correction
gives a positive-probability obstruction to \(\rho_n\to0\).  On the
endpoint/lower-subband critical regimes, the same correction dominates with
probability tending to one, giving the stronger high-probability obstruction.

Finally, the rate bounds use growing-degree CG non-stopping estimates.  For
\(p=1\), if
\[
a_n\to\infty,
\qquad
\log a_n=o(\log n),
\]
then
\[
\mathbb P\!\left(T_{\mathrm{CG},n}(\epsilon,1)\le a_n\right)\to0.
\]
For the \(0<p<1\) zero-ratio regimes, if
\[
b_n\to\infty,
\qquad
b_n\log(1+b_n)=o(\log n),
\]
then
\[
\mathbb P\!\left(T_{\mathrm{CG},n}(\epsilon,p)\le b_n\right)\to0.
\]
The stochastic upper bounds follow from the event inclusion
\[
\left\{
\rho_n(\epsilon,p)>
\frac{B\,\mathrm{budget}_n}{d_n}
\right\}
\subset
\left\{
\frac{T_{\mathrm{RCD},n}(\epsilon,p;I)}{n}>
B\,\mathrm{budget}_n
\right\}
\cup
\left\{
T_{\mathrm{CG},n}(\epsilon,p)\le d_n
\right\}.
\]
Here \(\mathrm{budget}_n=\log n\) and \(d_n=a_n\) for \(p=1\), while
\(\mathrm{budget}_n=1\) and \(d_n=b_n\) for \(0<p<1\).  This proves the rate
bounds stated in the theorem.

Full details are given in Appendix~\ref{app:cg-proof}.

\subsubsection{Trajectory analysis}

\begin{figure}[htbp]
    \centering
    \includegraphics[width=0.92\linewidth]{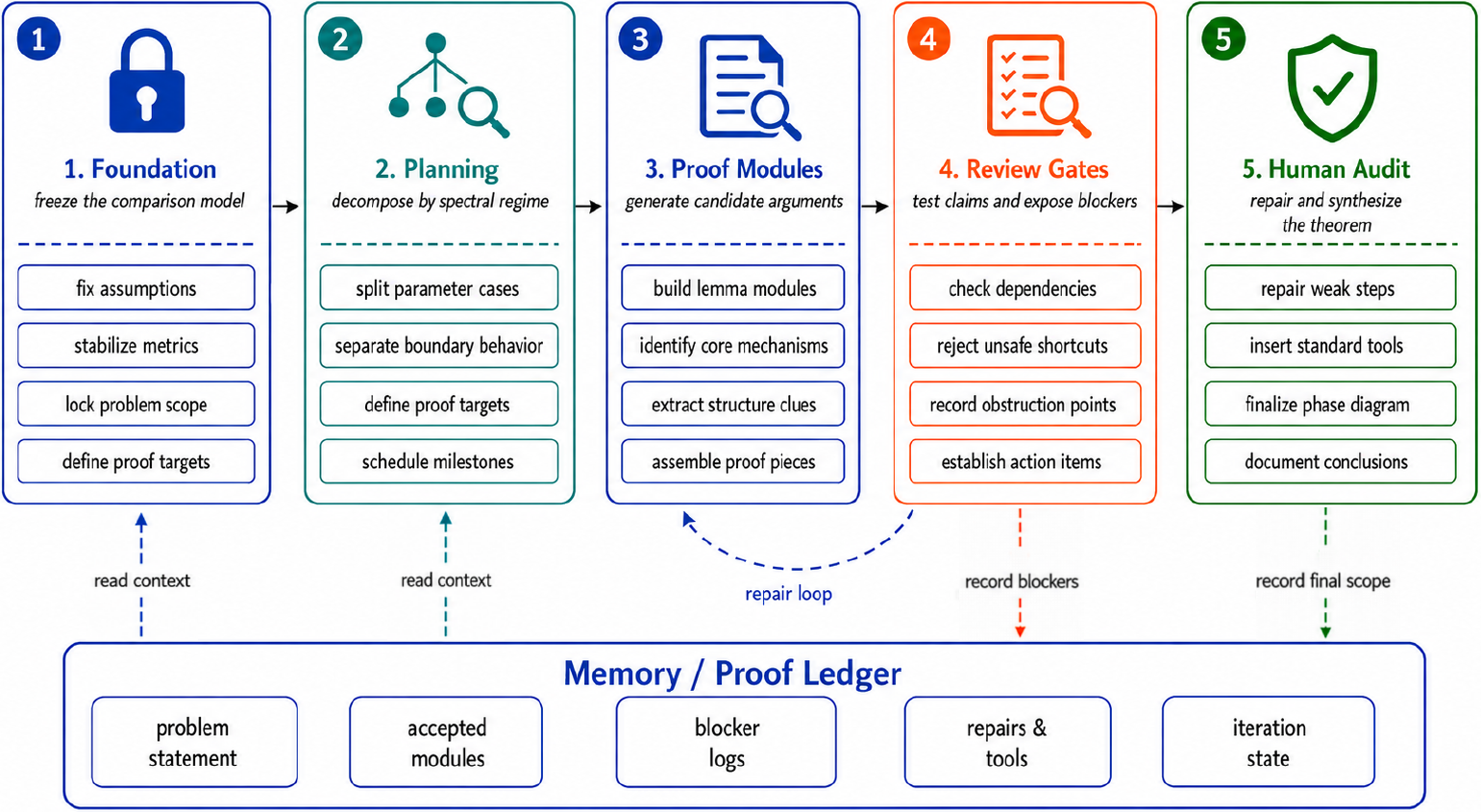}
    \caption{\AIVN{} trajectory for Problem I.}
    \label{fig:cg-rcd-trajectory}
\end{figure}

Figure~\ref{fig:cg-rcd-trajectory} summarizes the agent-collaboration trajectory
that led from the original open problem to the final
fixed-parameter phase diagram.  The figure records the main mathematical roles
played by different agents in the \AIVN{} framework: source auditing and
convention freezing, candidate proof generation, independent review, blocker
identification, proof repair, and final human mathematical audit.

\paragraph{1. Foundation and model freeze.}
The foundation stage fixed the exact comparison model from Problem~2.4 in the
Simons workshop collection: Haar eigenvectors, power-law eigenvalues
\(\lambda_j=j^{-p}\), source right-hand side \(b_n=A_nz_n\), exact
unpreconditioned CG, and source RCD counted in epochs.  This stage also fixed
the comparison ratio \(\rho_n(\varepsilon,p)\) and prevented later agents from
silently changing the solver model, the source distribution, or the unit of
comparison.

\paragraph{2. Orchestrated decomposition.}
The plan agent decomposed the original question into spectral regimes:
the summable row \(p>1\), the logarithmic row \(p=1\), and the nonsummable
range \(0<p<1\).  In the nonsummable range, it further separated the
below-floor, above-floor, and critical-boundary cases.  This orchestration
turned the original open-ended comparison into smaller proof targets with
explicit dependencies and review gates.

\paragraph{3. Candidate proof generation.}
Proof agents generated separate candidate modules for the CG polynomial
representation, the RCD epoch upper bound, the RCD lower gate, the active
moment problem, the \(p=1\) logarithmic row, the \(p>1\) summable row, and the
critical Schur-complement boundary.  These modules produced several positive
mathematical outputs, including the identification of the active degree
\(K(p)\), the floor \(F(p)\), the finite-dimensional residual-polynomial
reduction, and the first-inactive Schur mechanism at the critical threshold.

\paragraph{4. Review gates and blockers.}
Review agents blocked
several overstrong shortcuts: fixed-degree CG non-stopping was not allowed to
imply growing-degree rate bounds; selected-row or proxy estimates were not
automatically upgraded to global statements; and the critical boundary could
not be justified by a compressed ``standard asymptotics'' step.  These blockers
forced the proof to split active-block estimates, Schur-gain estimates,
inverse-coordinate stability, endpoint cases, and growing-degree controls into
separate verifiable units.

\paragraph{5. Repair loop and final audit.}
The final appendix was produced only after these agent-generated modules were
manually audited and repaired.  In particular, informal or invalid arguments
were replaced by standard mathematical tools: matrix Chernoff bounds for random
Gram consistency, a projection-product theorem for the \(p>1\) Hilbert-space
contraction, explicit Schur-complement asymptotics at the critical boundary,
and a separate treatment of the endpoint \(K=0\), i.e. \(p=1/3\).  Thus the
pipeline contributed both constructive proof modules and negative information
about which proof routes were not justified.  

\medskip

The final proof in
Appendix~\ref{app:cg-proof} should therefore be viewed as a human-audited
synthesis of agent-generated and agent-reviewed proof modules. Two lessons from this trajectory are worth making explicit.

\paragraph{Direct GPT-Pro comparison.}
As a comparison point, we also submitted the same CG-versus-RCD question
directly to GPT-Pro in an advanced reasoning mode.\footnote{\url{https://chatgpt.com/s/t_6a19b814f7508191822a0b542bb88347}}
The direct response produced an attractive heuristic picture: it identified
reasonable spectral-filter intuition, proposed polynomial CG scalings, and
suggested that epoch-scaled RCD should often dominate.  However, it did not
close the proof.  In particular, it treated coordinate RCD too much like a
diagonal spectral filter, blurred the distinction between expected, typical,
and high-probability stopping times, and asserted sharp CG rates without the
growing-degree lower bounds needed to justify them.  This comparison clarifies
one role of the \AIVN{} framework: it does not merely ask for a more fluent
answer, but decomposes a plausible heuristic into auditable proof obligations.

\paragraph{Human steering and rate-claim repair.}

A second lesson concerns the role of human steering. The initial agentic proof trajectory focused on the qualitative phase diagram. From a human mathematical perspective, however, such a qualitative convergence statement naturally calls for a more refined analysis of the speed of convergence. The model did not initially make this quantitative question explicit; only after human intervention did the proof search turn toward rate bounds. 

Once this target was made explicit, early
candidate proofs suggested a stronger negative-power rate for
\(\rho_n(\varepsilon,p)\).  The subsequent review stage, again guided by the
human request to check what was actually proved, identified that this rate
claim was not supported by the available estimates.  The fixed-degree CG
argument showed only that $\mathbb P(T_{\mathrm{CG},n}\le d)\to0$ for every fixed \(d\), whereas a negative-power bound for \(\rho_n\) would
require a polynomial lower bound on \(T_{\mathrm{CG},n}\).  Subsequent
discussion and repair replaced this overstrong statement by a conservative
growing-degree non-stopping estimate.  The final theorem therefore states
bounds such as
\[
  \rho_n(\varepsilon,p)=O_{\mathbb P}(1/b_n),
  \qquad
  b_n\to\infty,\quad b_n\log(1+b_n)=o(\log n),
\]
rather than an unsupported fixed negative exponent.  This illustrates that the human intervention still matters: human mathematical judgment helped formulate
the right rate question, while the agentic review and repair loop converted an
overstrong candidate answer into a statement supported by the proof.

%% ======================================================================
\subsection{Problem II: QRCP orthonormal selection}

\paragraph{Background.}

QR factorization with column pivoting (QRCP), introduced in the Householder least-squares setting by Businger and Golub~\cite{BusingerGolub1965}, later became a standard deterministic tool for numerical rank detection, rank-revealing factorizations, and column-subset selection. Businger and Golub's greedy pivot rule selects, at each step, the column with largest residual norm after projection. Subsequent theoretical work clarified the distinction between the existence
of good rank-revealing column permutations and the behavior of particular
pivoting rules.  Hong and Pan~\cite{HongPan1992} gave column permutation existence
results for factorizations with provable singular-value bound, and Osinsky~\cite{Osinsky2025SingleSVD} gave a fairly efficient
$O(nk^2)$ algorithm achieving this bound. The Simons workshop collection~\cite{AmselEtAl2026OpenQuestions} \textit{Problem 4.3} asks whether the classical QRCP enforces a conditioning guarantee in the orthonormal-row setting. Theorem~\ref{thm:qrcp-source-scale-obstruction} answers this
question in the negative.

\medskip
The problem concerns the rank-revealing power of exact QRCP in the orthonormal regime. Let $Q\in\mathbb{R}^{n\times k}$ have orthonormal columns, and apply QRCP to $A=Q^T$. Equivalently, at each step one selects the remaining row label whose column in $A$ has largest squared residual after orthogonal projection away from the span already chosen. If $I_{\mathrm{piv}}=(i_1,\ldots,i_k)$ is the selected ordered set, write
\[
M(Q)=Q[I_{\mathrm{piv}},:],\qquad
\gamma(Q,I_{\mathrm{piv}})=\|M(Q)^{-1}\|_2
\]

when $M(Q)$ is nonsingular.  The open question asks whether the residual-pivot
rule forces $\gamma(Q,I_{\mathrm{piv}})$ to stay under the scale $\sqrt{k(n-k+1)}$.

\medskip
We prove a stronger obstruction: the residual-pivot rule can fail by an
arbitrarily large polynomial factor even when the row energies of $Q$ are
uniformly controlled.

\begin{theorem}[Bounded-coherence obstruction for QRCP]\label{thm:qrcp-source-scale-obstruction}
For every $H_*>1$, $B>0$, $e\ge0$, and $K\ge1$, there exist integers $k\ge K$
and $n>k$, and a matrix $Q\in\mathbb{R}^{n\times k}$ with $Q^TQ=I_k$, such
that QRCP applied to $Q^T$ selects an ordered pivot
set $I_{\mathrm{piv}}$ for which $M(Q)$ is nonsingular,
\[
\mu(Q)=\frac{n}{k}\max_{1\le i\le n}\|Q(i,:)\|_2^2<H_*,
\qquad
\frac{\gamma(Q,I_{\mathrm{piv}})}{\sqrt{k(n-k+1)}}>B k^e .
\]
\end{theorem}

Here \(\mu(Q)\) is the standard coherence of the subspace
\(\operatorname{range}(Q)\), i.e. the maximum row leverage score normalized by
the average leverage \(k/n\).  The condition \(\mu(Q)<H_*\) therefore places the
example in a bounded-coherence regime.  This rules out the explanation that the
QRCP failure is caused by a few high-leverage rows.
Such coherence conditions are ubiquitous in matrix completion and randomized
linear algebra~\cite{candes2012exact}.

\subsubsection{Proof sketch}

We first prescribe the block that QRCP will select, make this block badly conditioned, and only afterwards add extra columns to restore orthonormality.  Choose $m$ large, set $k=m+1$, and build
\[
A\in\mathbb{R}^{k\times n},
\qquad Q=A^T.
\]
The columns of $A$ are arranged in four blocks:
\[
A=\begin{pmatrix}R_+&X&F&G\end{pmatrix}.
\]
The first block is the selected block
\[
R_+=\begin{pmatrix}R_\beta&z\\0&\eta_p\end{pmatrix}.
\]
Here $R_\beta$ is an explicit upper-triangular chain.  Let
\[
\epsilon=e^{-\eta/k},
\qquad
s=(1-\epsilon^2)^{1/2},
\qquad
d_t=\frac1k\epsilon^{t-1}\quad(1\le t\le m),
\]
and choose signs $\sigma_1,\ldots,\sigma_m\in\{\pm1\}$.  Then
\[
(R_\beta)_{t,q}=
\begin{cases}
 d_t, & q=t,\\
 -\frac12\sigma_t\sigma_qs\,d_t, & t<q,\\
 0, & t>q.
\end{cases}
\]
The final selected column is $\begin{pmatrix}z\\ \eta_p\end{pmatrix}$, where
\[
z_t=\frac14\sigma_t s d_t\quad(t<m),
\qquad
z_m=\frac14\sigma_m\left(1-\frac1{16}\right)^{1/2}d_m,
\qquad
\eta_p=\frac14 d_m.
\]

The signs are chosen so that the entries of $R_\beta^{-1}z$ reinforce rather than cancel.  This triangular back-substitution is the source of the large inverse norm of the selected block.

The other three blocks only complete the identity $AA^T=I_k$, while keeping every added column too small to be selected by QRCP.  The final selected column creates a mixed term between the first $m$ coordinates and the last coordinate.  We cancel this term by adding two cross columns

\[
X=
\begin{pmatrix}
\theta_1z&\theta_2z\\
-b_1\eta_p&-b_2\eta_p
\end{pmatrix},
\qquad
\theta_1b_1+\theta_2b_2=1.
\]

The mixed block from $X X^T$ cancels exactly the mixed block from $R_+R_+^T$.

After this cancellation, the remaining missing matrix in the first $m$ coordinates is
\[
M_{\mathrm{rem}}
=I_m-R_\beta R_\beta^T-
\left(1+\theta_1^2+\theta_2^2\right)zz^T.
\]
For large $m$ this matrix is positive definite.  We split it into many small rank-one pieces,
\[
M_{\mathrm{rem}}=\sum_{h=1}^{N_{\mathrm{frame}}}W_hu_hu_h^T,
\qquad
0<W_h<d_m^2,
\qquad
\|u_h\|_2=1,
\]
and put these pieces into the pure frame block
\[
F=
\begin{pmatrix}
\sqrt{W_1}u_1&\sqrt{W_2}u_2&\cdots&\sqrt{W_{N_{\mathrm{frame}}}}u_{N_{\mathrm{frame}}}\\[1mm]
0&0&\cdots&0
\end{pmatrix}.
\]
Thus $F$ fills exactly the missing first-coordinate identity and contributes nothing to the last coordinate.  The bound $W_h<d_m^2$ is also the pivot-control condition: before the final step, every frame column has residual below the next intended selected residual, and after the first $m$ selected columns it has zero residual.

Only the last diagonal entry remains. We split the remaining scalar into positive pieces $e_1,\ldots,e_{N_{\mathrm{res}}}$ and set
\[
G=
\begin{pmatrix}
0&0&\cdots&0\\[1mm]
\sqrt{e_1}&\sqrt{e_2}&\cdots&\sqrt{e_{N_{\mathrm{res}}}}
\end{pmatrix}.
\]
These residual-only columns complete the lower-right entry of $AA^T$.

Putting the blocks together, $X$ cancels the mixed term, $F$ restores the first-coordinate identity, and $G$ restores the last diagonal entry.  Hence
\[
AA^T=I_k,
\qquad Q^TQ=I_k.
\]
At the same time, the small-column bounds force QRCP to select the prescribed columns corresponding to $R_+$, namely $I_{\mathrm{piv}}=(1,2,\ldots,m,m+1)$.

It remains to quantify the bad conditioning of the selected block.  Let $Z=\|R_\beta^{-1}z\|_2^2$.

Since
\[
R_+\begin{pmatrix}-R_\beta^{-1}z\\1\end{pmatrix}
=
\begin{pmatrix}0\\\eta_p\end{pmatrix},
\]
the selected block has a very small singular direction, and therefore
\[
\gamma(Q,I_{\mathrm{piv}})^2\ge \frac{Z}{\eta_p^2}.
\]
The chosen sign pattern in the triangular chain makes $Z$ grow exponentially which is faster than any fixed polynomial in $k$. The added columns increase the total row count, but only by the amount needed to complete the identity; this does not erase the growth coming from $Z$. Consequently, for sufficiently large $k$,
\[
\frac{\gamma(Q,I_{\mathrm{piv}})}{\sqrt{k(n-k+1)}}>Bk^e.
\]

Full details are given in Appendix~\ref{app:qrcp-proof}, where the coherence
bound \(\mu(Q)<H_*\) is also established.  The proof has also been verified in
Lean, with the aid of the autoformalization agent
\textit{Archon}\footnote{\url{https://github.com/frenzymath/Archon}}.

\medskip
Both results were discovered and verified through the \AIVN{} framework described in Section~\ref{sec:framework}.

\subsubsection{Trajectory analysis}

\begin{figure}[htbp]
    \centering
    \includegraphics[width=0.92\linewidth]{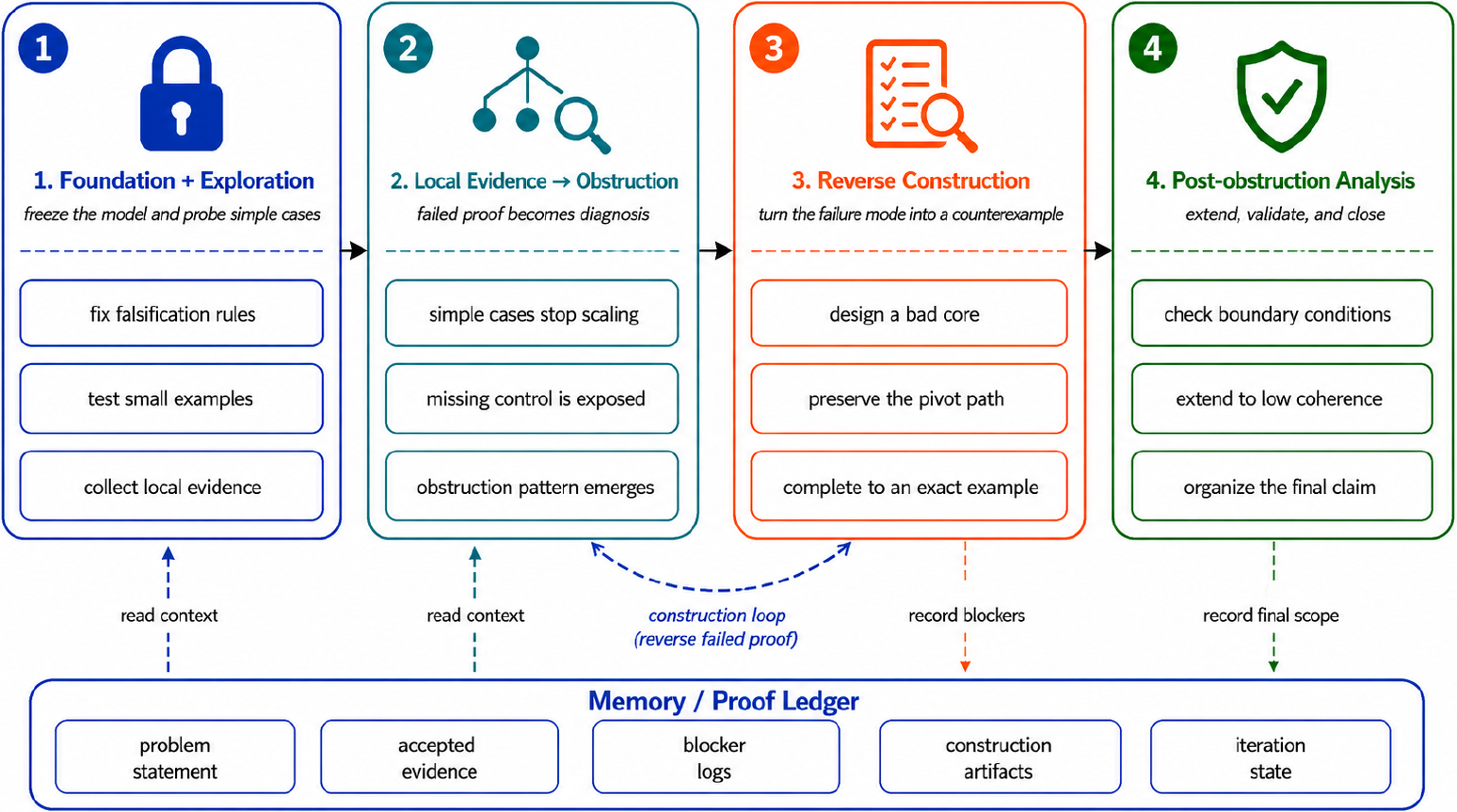}
    \caption{\AIVN{} trajectory for Problem II.}
    \label{fig:qrcp-trajectory}
\end{figure}

In this problem, Iteris converts failed proof attempts into targeted structural obstruction insights to construct a exact counterexample and automatically extends to the low coherence setting.

\paragraph{1. Model freeze and proof-oriented exploration.}
\AIVN{} first made the statement stable enough to prove or disprove:
foundation agents fixed the QRCP convention, tie rule, quality measure, and
falsification standard.  With this target frozen, proof agents attacked simple
cases ($k \le 2$), experiment agents ran bounded searches and replays, and review agents
kept the evidence scoped.

\paragraph{2. From partial positive evidence to a localized obstruction.}
The next phase was driven by the failure of these simple proof patterns to
scale ($k\ge 3$).  Small cases gave positive evidence, but the first genuinely nontrivial
setting ($k=3$) exposed a missing control: local residual norm and orthogonality do not by themselves control the cross interaction among the QRCP-selected rows. The failed proof attempt therefore
became diagnostic. It identified the structure that a counterexample
would need to exploit.

\paragraph{3. Reversing the failed proof into an exact counterexample.}
After this vulnerable structure was identified, the system shifted to
reverse construction. It designed a selected core satisfying the local
pivot conditions but violating the desired conditioning bound. 
Foundation work then turned the remaining issue into a completion
problem: add unused rows to satisfy global orthogonality while preserving the
intended QRCP pivot path. The counterexample thus
came from reversing a failed proof structure, not from blind numerical search.

\paragraph{4. Post-obstruction Analysis.}
After identifying the obstruction, the system did not limit its analysis to this single case. 
The foundation and proof agents evaluated whether the mechanism could satisfy stricter regularity requirements. This process automatically extends the theorem to the low coherence construction of the obstruction. 

\medskip

The framework contribution is the controlled handoff between modes.  File
memory kept the statement, evidence, and route status stable.  Foundation
froze definitions and turned proof gaps into construction targets; experiment
supplied diagnostics; proof exposed the obstruction; review controlled claim
promotion and route retirement.  The QRCP case is therefore a trajectory from
proof attempt, to obstruction insight and exact counterexample.

\paragraph{Direct GPT-Pro comparison.}
The direct response from GPT-Pro\footnote{\url{https://chatgpt.com/s/t_6a1bf4501af8819191189ec9676c058e}} produced a meaningful high-level counterexample sketch. It correctly identified the relevant QRCP failure mechanism: a triangular block embedded through an orthogonal completion, can force the selected submatrix to be badly conditioned. The limitation was instead paper-quality closure. Several essential steps were treated as plausible rather than discharged as explicit, auditable obligations. The baseline also did not push the construction through boundary checks such as coherence and leverage constraints. Thus the comparison should not be read as GPT-Pro missing the main idea; rather, it clarifies one role of \AIVN{}: turning a plausible counterexample mechanism into a closed, robust, and auditable construction.

\paragraph{Human audit and semantic alignment.}
The human role was primarily audit and organization. \AIVN{} supplied the
counterexample mechanism, intermediate construction artifacts, and boundary
tests.  Human work curated these outputs into a paper-level narrative and
checked semantic consistency.  After the formalization agent \textit{Archon} completed
Lean-level verification, the main human task was to align formal theorem statements and paper claims so that they referred to the same object and scope.
\section{Discussion}\label{sec:discussion}

Agentic research loops require reliable feedback, and computational mathematics is well suited because it offers both experimental and proof-based feedback. In both case studies, \AIVN{} used numerical evidence to narrow candidate proof routes and mathematical proof to confirm or reject them.  For instance, in the CG problem, numerical sweeps identified the Hankel floor as the key structural feature; in the QRCP construction, experiments isolated back-substitution pressure as the source-scale obstruction.  Neither type of evidence alone was sufficient---numerical results guided search, and proofs settled claims.

The results produced by \AIVN{} still required expert
auditing to distinguish correct arguments from unsupported steps, and they
required human reorganization to turn long and circuitous derivations
into readable mathematical proofs. The two case studies illustrate different kinds of human intervention. In the CG case (Problem I), \AIVN{} identified the main phase structure and
generated much of the proof skeleton. However, one branch of the rate analysis
in the \(p<1\) regime used a stronger assumption and led to an incorrect bound in that branch. Human inspection
detected the mismatch, and the argument was repaired through subsequent
human--AI interaction. In the
QRCP case (Problem II), the construction and obstruction identified by
\AIVN{} were correct, but the proof
was indirect and difficult to read. The final presentation therefore involved
substantial human restructuring of the argument. The resulting construction
has a sorry-free Lean~4 formalization of the core statement, but this does not
eliminate the need to check that the formal statement matches the informal
mathematical theorem. 

Thus, \AIVN{} shifts rather than removes human effort. In these case studies,
the system was most valuable in broadening the search space, finding candidate
structures, and producing draft arguments for human inspection. The human role
remained essential for verifying assumptions, detecting gaps, repairing
incorrect branches, and producing the final exposition. We therefore view
\AIVN{} as a tool for human-verified agentic research, not as a replacement for
mathematical judgment.

\section*{Acknowledgements}

The authors would like to warmly thank Zhangsong Li for carefully verifying the proof of CG-versus-sketching problem in Appendix~\ref{app:cg-proof}.

This work is supported in part by the Fundamental and Interdisciplinary Disciplines Breakthrough Plan of the Ministry of Education of China (JYB2025XDXM113), the National Key R\&D Program of China grant 2024YFA1014000, and the New Cornerstone Investigator Program.

\clearpage

\bibliographystyle{plainnat}
\bibliography{refs}

@misc{AbouzaidEtAl2026FirstProof,
      title={First Proof}, 
      author={Mohammed Abouzaid and Andrew J. Blumberg and Martin Hairer and Joe Kileel and Tamara G. Kolda and Paul D. Nelson and Daniel Spielman and Nikhil Srivastava and Rachel Ward and Shmuel Weinberger and Lauren Williams},
      year={2026},
      eprint={2602.05192},
      archivePrefix={arXiv},
      primaryClass={cs.AI},
      url={https://arxiv.org/abs/2602.05192}, 
}

@article{AmemiyaAndo1965,
  title={Convergence of random products of contractions in Hilbert space},
  author={Amemiya, Ichiro and Ando, Tadao},
  journal={Acta Scientiarum Mathematicarum},
  volume={26},
  number={3--4},
  pages={239--244},
  year={1965}
}

@misc{AmselEtAl2026OpenQuestions,
      title={Linear Systems and Eigenvalue Problems: Open Questions from a Simons Workshop}, 
      author={Noah Amsel and Yves Baumann and Paul Beckman and Peter Bürgisser and Chris Camaño and Tyler Chen and Edmond Chow and Anil Damle and Michal Derezinski and Mark Embree and Ethan N. Epperly and Robert Falgout and Mark Fornace and Anne Greenbaum and Chen Greif and Diana Halikias and Zhen Huang and Elias Jarlebring and Yiannis Koutis and Daniel Kressner and Rasmus Kyng and Jörg Liesen and Jackie Lok and Raphael A. Meyer and Yuji Nakatsukasa and Kate Pearce and Richard Peng and David Persson and Eliza Rebrova and Ryan Schneider and Rikhav Shah and Edgar Solomonik and Nikhil Srivastava and Alex Townsend and Robert J. Webber and Jess Williams},
      year={2026},
      eprint={2602.05394},
      archivePrefix={arXiv},
      primaryClass={math.NA},
      url={https://arxiv.org/abs/2602.05394}, 
}

@article{Boiko2024Coscientist, title={Autonomous chemical research with large language models}, volume={624}, ISSN={1476-4687}, url={http://dx.doi.org/10.1038/s41586-023-06792-0}, DOI={10.1038/s41586-023-06792-0}, number={7992}, journal={Nature}, publisher={Springer Science and Business Media LLC}, author={Boiko, Daniil A. and MacKnight, Robert and Kline, Ben and Gomes, Gabe}, year={2023}, month=Dec, pages={570–578} }

@article{BusingerGolub1965, title={Linear least squares solutions by householder transformations}, volume={7}, ISSN={0945-3245}, url={http://dx.doi.org/10.1007/bf01436084}, DOI={10.1007/bf01436084}, number={3}, journal={Numerische Mathematik}, publisher={Springer Science and Business Media LLC}, author={Businger, Peter and Golub, Gene H.}, year={1965}, month=June, pages={269–276} }

@article{candes2012exact, title={Exact Matrix Completion via Convex Optimization}, volume={9}, ISSN={1615-3383}, url={http://dx.doi.org/10.1007/s10208-009-9045-5}, DOI={10.1007/s10208-009-9045-5}, number={6}, journal={Foundations of Computational Mathematics}, publisher={Springer Science and Business Media LLC}, author={Candès, Emmanuel J. and Recht, Benjamin}, year={2009}, month=Apr, pages={717–772} }

@article{comanici2025gemini,
	title = {Gemini 2.5: Pushing the Frontier with Advanced Reasoning,
	         Multimodality, Long Context, and Next Generation Agentic
	         Capabilities},
	author = {Comanici, Gheorghe and Bieber, Eric and Schaekermann, Mike and
	          Pasupat, Ice and Sachdeva, Noveen and Dhillon, Inderjit and
	          Blistein, Marcel and Ram, Ori and Zhang, Dan and Rosen, Evan
	          and others},
	journal = {arXiv preprint arXiv:2507.06261},
	year = {2025},
}

@article{Davies2021Guiding, title={Advancing mathematics by guiding human intuition with AI}, volume={600}, ISSN={1476-4687}, url={http://dx.doi.org/10.1038/s41586-021-04086-x}, DOI={10.1038/s41586-021-04086-x}, number={7887}, journal={Nature}, publisher={Springer Science and Business Media LLC}, author={Davies, Alex and Veličković, Petar and Buesing, Lars and Blackwell, Sam and Zheng, Daniel and Tomašev, Nenad and Tanburn, Richard and Battaglia, Peter and Blundell, Charles and Juhász, András and Lackenby, Marc and Williamson, Geordie and Hassabis, Demis and Kohli, Pushmeet}, year={2021}, month=Dec, pages={70–74} }

@article{DerezinskiRebrova2024, title={Sharp Analysis of Sketch-and-Project Methods via a Connection to Randomized Singular Value Decomposition}, volume={6}, ISSN={2577-0187}, url={http://dx.doi.org/10.1137/23m1545537}, DOI={10.1137/23m1545537}, number={1}, journal={SIAM Journal on Mathematics of Data Science}, publisher={Society for Industrial & Applied Mathematics (SIAM)}, author={Dereziński, Michał and Rebrova, Elizaveta}, year={2024}, month=Feb, pages={127–153} }

@misc{FengEtAl2026AutonomousMath,
      title={Towards Autonomous Mathematics Research}, 
      author={Tony Feng and Trieu H. Trinh and Garrett Bingham and Dawsen Hwang and Yuri Chervonyi and Junehyuk Jung and Joonkyung Lee and Carlo Pagano and Sang-hyun Kim and Federico Pasqualotto and Sergei Gukov and Jonathan N. Lee and Junsu Kim and Kaiying Hou and Golnaz Ghiasi and Yi Tay and YaGuang Li and Chenkai Kuang and Yuan Liu and Hanzhao Lin and Evan Zheran Liu and Nigamaa Nayakanti and Xiaomeng Yang and Heng-Tze Cheng and Demis Hassabis and Koray Kavukcuoglu and Quoc V. Le and Thang Luong},
      year={2026},
      eprint={2602.10177},
      archivePrefix={arXiv},
      primaryClass={cs.LG},
      url={https://arxiv.org/abs/2602.10177}, 
}

@article{GowerRichtarik2015, title={Randomized Iterative Methods for Linear Systems}, volume={36}, ISSN={1095-7162}, url={http://dx.doi.org/10.1137/15m1025487}, DOI={10.1137/15m1025487}, number={4}, journal={SIAM Journal on Matrix Analysis and Applications}, publisher={Society for Industrial & Applied Mathematics (SIAM)}, author={Gower, Robert M. and Richtárik, Peter}, year={2015}, month=Jan, pages={1660–1690} }

@article{guo2025deepseek,
	title = {DeepSeek-R1 Incentivizes Reasoning in LLMs Through
	         Reinforcement Learning},
	author = {Guo, Daya and Yang, Dejian and Zhang, Haowei and Song, Junxiao
	          and Wang, Peiyi and Zhu, Qihao and Xu, Runxin and Zhang, Ruoyu
	          and Ma, Shirong and Bi, Xiao and others},
	journal = {Nature},
	volume = {645},
	number = {8081},
	pages = {633--638},
	year = {2025},
	publisher = {Nature Publishing Group UK London},
}

@article{HestenesStiefel1952, title={Methods of conjugate gradients for solving linear systems}, volume={49}, ISSN={0091-0635}, url={http://dx.doi.org/10.6028/jres.049.044}, DOI={10.6028/jres.049.044}, number={6}, journal={Journal of Research of the National Bureau of Standards}, publisher={National Institute of Standards and Technology (NIST)}, author={Hestenes, M.R. and Stiefel, E.}, year={1952}, month=Dec, pages={409} }

@article{HongPan1992, title={Rank-Revealing QR Factorizations and the Singular Value Decomposition}, volume={58}, ISSN={0025-5718}, url={http://dx.doi.org/10.2307/2153029}, DOI={10.2307/2153029}, number={197}, journal={Mathematics of Computation}, publisher={JSTOR}, author={Hong, Y. P. and Pan, C.-T.}, year={1992}, month=Jan, pages={213} }

@misc{JuEtAl2026ConjectureResolution,
      title={Automated Conjecture Resolution with Formal Verification}, 
      author={Haocheng Ju and Guoxiong Gao and Jiedong Jiang and Bin Wu and Zeming Sun and Leheng Chen and Yutong Wang and Yuefeng Wang and Zichen Wang and Wanyi He and Peihao Wu and Liang Xiao and Ruochuan Liu and Bryan Dai and Bin Dong},
      year={2026},
      eprint={2604.03789},
      archivePrefix={arXiv},
      primaryClass={cs.LG},
      url={https://arxiv.org/abs/2604.03789}, 
}

@article{LeventhalLewis2010, title={Randomized Methods for Linear Constraints: Convergence Rates and Conditioning}, volume={35}, ISSN={1526-5471}, url={http://dx.doi.org/10.1287/moor.1100.0456}, DOI={10.1287/moor.1100.0456}, number={3}, journal={Mathematics of Operations Research}, publisher={Institute for Operations Research and the Management Sciences (INFORMS)}, author={Leventhal, D. and Lewis, A. S.}, year={2010}, month=Aug, pages={641–654} }

@misc{LokRebrova2025,
      title={Subspace-constrained randomized coordinate descent for linear systems with good low-rank matrix approximations}, 
      author={Jackie Lok and Elizaveta Rebrova},
      year={2026},
      eprint={2506.09394},
      archivePrefix={arXiv},
      primaryClass={math.NA},
      url={https://arxiv.org/abs/2506.09394}, 
}

@article{LuEtAl2026AIScientist, title={Towards end-to-end automation of AI research}, volume={651}, ISSN={1476-4687}, url={http://dx.doi.org/10.1038/s41586-026-10265-5}, DOI={10.1038/s41586-026-10265-5}, number={8107}, journal={Nature}, publisher={Springer Science and Business Media LLC}, author={Lu, Chris and Lu, Cong and Lange, Robert Tjarko and Yamada, Yutaro and Hu, Shengran and Foerster, Jakob and Ha, David and Clune, Jeff}, year={2026}, month=Mar, pages={914–919} }

@article{Nesterov2012, title={Efficiency of Coordinate Descent Methods on Huge-Scale Optimization Problems}, volume={22}, ISSN={1095-7189}, url={http://dx.doi.org/10.1137/100802001}, DOI={10.1137/100802001}, number={2}, journal={SIAM Journal on Optimization}, publisher={Society for Industrial & Applied Mathematics (SIAM)}, author={Nesterov, Yu.}, year={2012}, month=Jan, pages={341–362} }

@misc{NovikovEtAl2025AlphaEvolve,
      title={AlphaEvolve: A coding agent for scientific and algorithmic discovery}, 
      author={Alexander Novikov and Ngân Vũ and Marvin Eisenberger and Emilien Dupont and Po-Sen Huang and Adam Zsolt Wagner and Sergey Shirobokov and Borislav Kozlovskii and Francisco J. R. Ruiz and Abbas Mehrabian and M. Pawan Kumar and Abigail See and Swarat Chaudhuri and George Holland and Alex Davies and Sebastian Nowozin and Pushmeet Kohli and Matej Balog},
      year={2025},
      eprint={2506.13131},
      archivePrefix={arXiv},
      primaryClass={cs.AI},
      url={https://arxiv.org/abs/2506.13131}, 
}

@article{Osinsky2025SingleSVD, title={Close to optimal column approximation using a single SVD}, volume={725}, ISSN={0024-3795}, url={http://dx.doi.org/10.1016/j.laa.2025.07.016}, DOI={10.1016/j.laa.2025.07.016}, journal={Linear Algebra and its Applications}, publisher={Elsevier BV}, author={Osinsky, A.I.}, year={2025}, month=Nov, pages={359–377} }

@article{RomeraParedes2024FunSearch, title={Mathematical discoveries from program search with large language models}, volume={625}, ISSN={1476-4687}, url={http://dx.doi.org/10.1038/s41586-023-06924-6}, DOI={10.1038/s41586-023-06924-6}, number={7995}, journal={Nature}, publisher={Springer Science and Business Media LLC}, author={Romera-Paredes, Bernardino and Barekatain, Mohammadamin and Novikov, Alexander and Balog, Matej and Kumar, M. Pawan and Dupont, Emilien and Ruiz, Francisco J. R. and Ellenberg, Jordan S. and Wang, Pengming and Fawzi, Omar and Kohli, Pushmeet and Fawzi, Alhussein}, year={2023}, month=Dec, pages={468–475} }

@book{Saad2003, title={Iterative Methods for Sparse Linear Systems}, ISBN={9780898718003}, url={http://dx.doi.org/10.1137/1.9780898718003}, DOI={10.1137/1.9780898718003}, publisher={Society for Industrial and Applied Mathematics}, author={Saad, Yousef}, year={2003}, month=Jan }

@article{yang2025qwen3,
  title={Qwen3 technical report},
  author={Yang, An and Li, Anfeng and Yang, Baosong and Zhang, Beichen and Hui, Binyuan and Zheng, Bo and Yu, Bowen and Gao, Chang and Huang, Chengen and Lv, Chenxu and others},
  journal={arXiv preprint arXiv:2505.09388},
  year={2025}
}

@misc{YaoEtAl2023ReAct,
      title={ReAct: Synergizing Reasoning and Acting in Language Models}, 
      author={Shunyu Yao and Jeffrey Zhao and Dian Yu and Nan Du and Izhak Shafran and Karthik Narasimhan and Yuan Cao},
      year={2023},
      eprint={2210.03629},
      archivePrefix={arXiv},
      primaryClass={cs.CL},
      url={https://arxiv.org/abs/2210.03629}, 
}

@misc{ZhengEtAl2026AICoMathematician,
      title={AI co-mathematician: Accelerating mathematicians with agentic AI}, 
      author={Daniel Zheng and Ingrid von Glehn and Yori Zwols and Iuliya Beloshapka and Lars Buesing and Daniel M. Roy and Martin Wattenberg and Bogdan Georgiev and Tatiana Schmidt and Andrew Cowie and Fernanda Viegas and Dimitri Kanevsky and Vineet Kahlon and Hartmut Maennel and Sophia Alj and George Holland and Alex Davies and Pushmeet Kohli},
      year={2026},
      eprint={2605.06651},
      archivePrefix={arXiv},
      primaryClass={cs.AI},
      url={https://arxiv.org/abs/2605.06651}, 
}

\clearpage

\beginappendix
\section{Proof details for CG versus sketching}
\label{app:cg-proof}

This appendix proves Theorem~\ref{thm:cg-power-law-phase}.  The proof is
fixed-parameter throughout: \(p\) and \(\varepsilon\) are fixed before
\(n\to\infty\).  The rate bounds proved below are conservative upper bounds;
no matching lower bounds or sharpness claims are made.

\subsection{Finite-dimensional reductions}

Let
\[
  A_n=U_n\operatorname{diag}(1^{-p},\ldots,n^{-p})U_n^*,
  \qquad
  b_n=A_nz_n,
  \qquad
  x_{*,n}=z_n,
\]
and let \(\alpha=U_n^*z_n\).  For \(d\ge0\), set
\[
  \mathcal P_d^1=\{q\in\mathbb R[\xi]:\deg q\le d,\ q(0)=1\}
\]
and
\[
  M_{d,n}(p)=
  \min_{q\in\mathcal P_d^1}
  \frac{\sum_{j=1}^n j^{-p}|q(j^{-p})|^2|\alpha_j|^2}
       {\sum_{j=1}^n j^{-p}|\alpha_j|^2}.
\]

\begin{proposition}[CG polynomial representation]
\label{prop:cg-poly-app}
Exact CG from \(x_0=0\) satisfies
\[
  T_{\mathrm{CG},n}(\varepsilon,p)
  =
  \min\{d\ge0:M_{d,n}(p)\le\varepsilon\}.
\]
In particular \(M_{0,n}(p)=1\), so \(T_{\mathrm{CG},n}(\varepsilon,p)\ge1\)
for \(\varepsilon\in(0,1)\).
\end{proposition}

\begin{proof}
The \(d\)-th CG iterate minimizes the \(A_n\)-error over
\[
  \operatorname{span}\{b_n,A_nb_n,\ldots,A_n^{d-1}b_n\}.
\]
Since \(b_n=A_nz_n\), every vector in this space is \(A_ns(A_n)z_n\) with
\(\deg s\le d-1\).  Hence the residual is
\[
  z_n-A_ns(A_n)z_n=q(A_n)z_n,
  \qquad q(\xi)=1-\xi s(\xi),
\]
where \(q(0)=1\) and \(\deg q\le d\).  Conversely, every such \(q\) has this
form because \(1-q(\xi)\) is divisible by \(\xi\).  Diagonalizing \(A_n\)
gives exactly the displayed minimum.
\end{proof}

The Haar vector satisfies
\[
  (|\alpha_1|^2,\ldots,|\alpha_n|^2)
  \stackrel d=
  \frac{(E_1,\ldots,E_n)}{\sum_iE_i},
  \tag{A.1}
\]
where the \(E_i\) are independent exponential variables of mean one.  The
common denominator cancels in \(M_{d,n}\).

For \(0<p<1\), define
\[
  Y_{j,n}=\left(\frac nj\right)^p,
  \qquad
  S_{a,n}=\sum_{j=1}^nY_{j,n}^aE_j,
  \qquad
  \mu_n=\frac1{S_{1,n}}\sum_{j=1}^nY_{j,n}E_j\,\delta_{Y_{j,n}}.
\]
After the change of variables \(r(y)=q(n^{-p}y)\),
\[
  M_{d,n}(p)
  =
  \min_{\deg r\le d,\ r(0)=1}\int r(y)^2\,\mu_n(dy).
  \tag{A.2}
\]

\begin{proposition}[RCD budget bounds]
\label{prop:rcd-budget-app}
Let \(A\succ0\) be deterministic, let \(x_\star=A^{-1}b\), and run RCD from
an arbitrary initial point \(x_0\).  Write
\[
u_t=x_t-x_\star,
\qquad
V_t=\|u_t\|_A^2,
\]
and define
\[
T_{\mathrm{RCD}}(\varepsilon)
=
\min\{t\ge0:V_t\le\varepsilon V_0\}.
\]
At each step, RCD samples \(i_t=i\) with probability
\[
\mathbb P(i_t=i\mid A)=\frac{A(i,i)}{\operatorname{tr}(A)}
\]
and performs the exact coordinate update.  Then, for every \(k\ge0\),
\[
  \mathbb P_I\!\left(T_{\mathrm{RCD}}(\varepsilon)>k\mid A,u_0\right)
  \le
  \varepsilon^{-1}
  \left(1-\frac{\lambda_{\min}(A)}{\operatorname{tr}(A)}\right)^k,
  \tag{A.3}
\]
where \(\mathbb P_I\) denotes probability over the RCD coordinate-sampling
sequence only.  Consequently, in the power-law model,
\[
  T_{\mathrm{RCD},n}(\varepsilon,p;I)/n=O_{\mathbb P}(1),
  \qquad 0<p<1,
  \tag{A.4}
\]
and
\[
  T_{\mathrm{RCD},n}(\varepsilon,1;I)/n=O_{\mathbb P}(\log n).
  \tag{A.5}
\]
\end{proposition}

\begin{proof}
Since \(u_t=x_t-x_\star\), the exact coordinate update gives
\[
  u_{t+1}
  =
  u_t-\frac{(Au_t)_{i_t}}{A(i_t,i_t)}\,\mathbf e_{i_t},
\]
where \(\mathbf e_i\) is the \(i\)-th coordinate vector.  Therefore
\[
  V_{t+1}
  =
  V_t-\frac{|(Au_t)_{i_t}|^2}{A(i_t,i_t)}.
\]
Taking expectation over \(i_t\), conditionally on \(A\) and \(u_t\), gives
\[
  \mathbb E_I[V_{t+1}\mid A,u_t]
  =
  V_t-\sum_i
  \frac{A(i,i)}{\operatorname{tr}(A)}
  \frac{|(Au_t)_i|^2}{A(i,i)}
  =
  V_t-\frac{\|Au_t\|_2^2}{\operatorname{tr}(A)}.
\]
Since
\[
  \|Au_t\|_2^2
  =
  u_t^*A^2u_t
  \ge
  \lambda_{\min}(A)u_t^*Au_t
  =
  \lambda_{\min}(A)V_t,
\]
we get
\[
  \mathbb E_I[V_{t+1}\mid A,u_t]
  \le
  \left(
  1-\frac{\lambda_{\min}(A)}{\operatorname{tr}(A)}
  \right)V_t.
\]
Iterating,
\[
  \mathbb E_I[V_k\mid A,u_0]
  \le
  \left(
  1-\frac{\lambda_{\min}(A)}{\operatorname{tr}(A)}
  \right)^k V_0.
\]
By Markov's inequality,
\[
  \mathbb P_I(T_{\mathrm{RCD}}(\varepsilon)>k\mid A,u_0)
  =
  \mathbb P_I(V_k>\varepsilon V_0\mid A,u_0)
  \le
  \varepsilon^{-1}
  \left(
  1-\frac{\lambda_{\min}(A)}{\operatorname{tr}(A)}
  \right)^k.
\]

For \(A=A_n\), one has
\[
\lambda_{\min}(A_n)=n^{-p},
\qquad
\operatorname{tr}(A_n)=H_{n,p}:=\sum_{j=1}^n j^{-p}.
\]
If \(0<p<1\), then
\[
  H_{n,p}\sim\frac{n^{1-p}}{1-p},
  \qquad
  \frac{n^{-p}}{H_{n,p}}
  =
  \frac{1-p}{n}(1+o(1)).
\]
Taking \(k=\lfloor Bn\rfloor\) gives an exponentially decaying upper tail in
\(B\), hence (A.4).  If \(p=1\), then \(H_{n,1}\sim\log n\).  Taking
\(k=\lfloor Bn\log n\rfloor\) gives (A.5).
\end{proof}

\begin{proposition}[RCD lower gate]
\label{prop:rcd-lower-app}
Fix \(0<p<1/3\) and \(\varepsilon\in(0,1)\).  If
\(\gamma\in(0,1)\) satisfies
\[
1-\gamma^{1-p}>\varepsilon,
\]
then there exists \(a>0\) such that
\[
  \mathbb P\!\left(T_{\mathrm{RCD},n}(\varepsilon,p;I)/n>a\right)\to1.
\]
\end{proposition}

\begin{proof}
Let \(a=\gamma/2\) and \(k_n=\lfloor an\rfloor\).  Write
\(I=(i_0,i_1,\ldots)\) for the coordinate-sampling sequence, and define the
selected-coordinate set
\[
S_{k_n}(I)=\{i_0,\ldots,i_{k_n-1}\}.
\]
Since \(x_0=0\) and each RCD step changes only one coordinate, the support of
\(x_{k_n}\) is contained in \(S_{k_n}(I)\).  In particular,
\[
|S_{k_n}(I)|\le k_n\le \gamma n.
\]

For a deterministic coordinate set \(S\subset\{1,\ldots,n\}\), define
\[
E_S=\operatorname{span}\{\mathbf e_i:i\in S\},
\]
where \(\mathbf e_i\) is the \(i\)-th coordinate vector.  Let \(P_{A_n,S}\) be
the Euclidean orthogonal projection onto the subspace \(A_n^{1/2}E_S\), and set
\[
  B_{A_n,S}
  =
  A_n^{1/2}(I-P_{A_n,S})A_n^{1/2}.
\]
For every \(x\in E_S\),
\[
  \|x-z_n\|_{A_n}^2
  =
  \|A_n^{1/2}x-A_n^{1/2}z_n\|_2^2
  \ge
  z_n^*B_{A_n,S}z_n.
\]
Indeed, the right-hand side is the squared Euclidean distance from
\(A_n^{1/2}z_n\) to the subspace \(A_n^{1/2}E_S\).

We next lower-bound \(\operatorname{tr}(B_{A_n,S})\).  Since \(P_{A_n,S}\) is
an orthogonal projection of rank at most \(|S|\), the Ky Fan maximum principle
gives
\[
  \operatorname{tr}(P_{A_n,S}A_n)
  \le
  \sum_{j=1}^{|S|}j^{-p}.
\]
Equivalently, this follows from the variational fact that among all
rank-\(m\) orthogonal projections \(P\), the maximum of
\(\operatorname{tr}(PA_n)\) is the sum of the largest \(m\) eigenvalues of
\(A_n\).  Therefore, if \(|S|\le\gamma n\),
\[
  \operatorname{tr}(B_{A_n,S})
  =
  \operatorname{tr}(A_n)-\operatorname{tr}(P_{A_n,S}A_n)
  \ge
  H_{n,p}-H_{\lfloor\gamma n\rfloor,p},
\]
where \(H_{m,p}=\sum_{j=1}^m j^{-p}\).  Since
\[
  \frac{H_{\lfloor\gamma n\rfloor,p}}{H_{n,p}}
  \to
  \gamma^{1-p},
\]
the assumption \(1-\gamma^{1-p}>\varepsilon\) implies that there exists
\(\theta>0\) such that, for all sufficiently large \(n\) and all
\(|S|\le\gamma n\),
\[
  \operatorname{tr}(B_{A_n,S})
  \ge
  (\varepsilon+\theta)H_{n,p}.
\]

Now define
\[
  C_{A_n,S}=B_{A_n,S}-\varepsilon A_n.
\]
Then
\[
  \operatorname{tr}(C_{A_n,S})\ge \theta H_{n,p}.
\]
Also \(0\le B_{A_n,S}\le A_n\), so
\[
  -\varepsilon A_n\le C_{A_n,S}\le (1-\varepsilon)A_n.
\]
Since \(\|A_n\|=1\), this gives
\[
  \|C_{A_n,S}\|\le1.
\]

For a uniform complex unit vector \(z_n\), the standard spherical quadratic
form identities give, conditionally on \(A_n\) and \(S\),
\[
  \mathbb E\!\left[z_n^*C_{A_n,S}z_n\mid A_n,S\right]
  =
  \frac{\operatorname{tr}(C_{A_n,S})}{n}
  \ge
  c\,n^{-p}
\]
for some \(c>0\), and
\[
  \operatorname{Var}\!\left(z_n^*C_{A_n,S}z_n\mid A_n,S\right)
  =
  \frac{
  \operatorname{tr}(C_{A_n,S}^2)
  -\operatorname{tr}(C_{A_n,S})^2/n
  }{n(n+1)}
  \le
  \frac1n.
\]
Chebyshev's inequality therefore yields, uniformly over all
\(|S|\le\gamma n\),
\[
  \mathbb P\!\left(
  z_n^*C_{A_n,S}z_n\le0
  \mid A_n,S
  \right)
  \le
  C n^{2p-1}
  \to0.
\]
Here \(p<1/3\) is more than enough.

Finally, conditional on \(A_n\) and the coordinate-sampling sequence \(I\), the
set \(S_{k_n}(I)\) is fixed and is independent of \(z_n\).  Since
\(|S_{k_n}(I)|\le\gamma n\), the preceding bound applies with
\(S=S_{k_n}(I)\).  Hence, with probability tending to one,
\[
  z_n^*B_{A_n,S_{k_n}(I)}z_n
  >
  \varepsilon z_n^*A_nz_n.
\]
Because \(x_{k_n}\in E_{S_{k_n}(I)}\), we have
\[
  V_{k_n}
  =
  \|x_{k_n}-z_n\|_{A_n}^2
  \ge
  z_n^*B_{A_n,S_{k_n}(I)}z_n
  >
  \varepsilon z_n^*A_nz_n
  =
  \varepsilon V_0.
\]
Thus \(T_{\mathrm{RCD},n}(\varepsilon,p;I)>k_n\) with probability tending to
one.  Since \(k_n/n\to a\), the claim follows.
\end{proof}

\subsection{The active floor}

For \(0<p<1\), set
\[
  K(p)=\max\{k\ge0:(2k+1)p<1\},
\]
and
\[
  \nu_p(dy)=\frac{1-p}{p}y^{-1/p}\,dy,
  \qquad y\in[1,\infty).
\]
For \(D\le K(p)\), define
\[
  F_D(p)=
  \min_{\deg r\le D,\ r(0)=1}
  \int r(y)^2\,\nu_p(dy),
  \qquad
  F(p)=F_{K(p)}(p).
\]
Equivalently, if
\[
  m_s(p)=\int y^s\,\nu_p(dy)=\frac{1-p}{1-(s+1)p},
  \qquad 0\le s\le 2D,
\]
and
\[
  G_D(p)=\bigl(m_{a+b}(p)\bigr)_{0\le a,b\le D},
\]
then \(G_D(p)\) is positive definite.  Indeed, for any nonzero coefficient
vector \(c=(c_0,\ldots,c_D)\),
\[
  c^*G_D(p)c
  =
  \int \left|\sum_{a=0}^D c_a y^a\right|^2\,\nu_p(dy)>0,
\]
because \(\nu_p\) has infinite support on \([1,\infty)\).  Hence \(G_D(p)\)
is invertible, and the constrained minimum is
\[
  F_D(p)=\frac1{(G_D(p)^{-1})_{00}}.
  \tag{A.6}
\]

\begin{lemma}[Moment convergence]
\label{lem:moment-app}
If \((m+1)p<1\), then
\[
  \int y^m\,\mu_n(dy)
  =
  \frac{S_{m+1,n}}{S_{1,n}}
  \xrightarrow{\mathbb P}
  \frac{1-p}{1-(m+1)p}.
  \tag{A.7}
\]
Moreover, on every fixed interval \([1,R]\), \(\mu_n\) converges in moments
to \(\nu_p\).
\end{lemma}

\begin{proof}
For \(a p<1\),
\[
  S_{a,n}=n^{ap}\sum_{j=1}^n j^{-ap}E_j,
  \qquad
  \mathbb E S_{a,n}\sim \frac n{1-ap},
  \qquad
  \operatorname{Var}(S_{a,n})=o(n^2).
\]
Hence \(S_{a,n}/n\to(1-ap)^{-1}\) in probability.  Taking \(a=m+1\) and
\(a=1\) proves (A.7).  The local statement follows from the same law of large
numbers on the index range \(nR^{-1/p}\le j\le n\), followed by the change of
variables \(y=t^{-p}\).
\end{proof}

\begin{proposition}[Fixed-degree CG limit]
\label{prop:fixed-degree-app}
For fixed \(0<p<1\) and fixed \(d\),
\[
  M_{d,n}(p)\xrightarrow{\mathbb P}F_{\min(d,K(p))}(p).
  \tag{A.8}
\]
\end{proposition}

\begin{proof}
Let \(D=\min(d,K(p))\).  The upper bound follows by inserting the minimizer
for \(F_D(p)\) into (A.2) and using Lemma~\ref{lem:moment-app}.  For the lower
bound, take a minimizer \(r_n\) in (A.2).  Since \(r\equiv1\) is feasible,
\(\int r_n^2\,d\mu_n\le1\).  Moment convergence on \([1,2]\) gives a positive
definite local Gram limit, hence the coefficients of \(r_n\) are tight.  Any
coefficient limit \(r\) cannot have degree larger than \(K(p)\): if
\(\deg r=\ell>K(p)\), then \((2\ell+1)p\ge1\), so
\(\int_1^R r^2\,d\nu_p\to\infty\), contradicting the bound
\(\int r_n^2\,d\mu_n\le1\) after taking \(R\) fixed and large.  Therefore
every limit point has degree at most \(D\), and local convergence followed by
\(R\to\infty\) gives the lower bound \(F_D(p)\).
\end{proof}

\subsection{Qualitative phase diagram}

\begin{proposition}[Zero ratio below the floor]
\label{prop:below-floor-app}
If \(0<p<1\) and \(\varepsilon<F(p)\), then
\[
  \rho_n(\varepsilon,p)\xrightarrow{\mathbb P}0.
\]
In particular, if \(1/3\le p<1\), then \(K(p)=0\), \(F(p)=1\), and the same
conclusion holds for every \(\varepsilon\in(0,1)\).
\end{proposition}

\begin{proof}
For every fixed \(d\), Proposition~\ref{prop:fixed-degree-app} gives
\[
  \mathbb P(T_{\mathrm{CG},n}(\varepsilon,p)\le d)
  =
  \mathbb P(M_{d,n}(p)\le\varepsilon)
  \to0.
\]
Thus \(T_{\mathrm{CG},n}\to\infty\) in probability.  By (A.4),
\(T_{\mathrm{RCD},n}/n=O_{\mathbb P}(1)\).  Hence
\[
  \mathbb P(\rho_n>\eta)
  \le
  \mathbb P(T_{\mathrm{RCD},n}/n>C)
  +
  \mathbb P(T_{\mathrm{CG},n}\le C/\eta),
\]
and the conclusion follows by first taking \(n\to\infty\), then \(C\to\infty\).
\end{proof}

\begin{proposition}[Strictly above the floor]
\label{prop:above-floor-app}
If \(0<p<1/3\) and \(F(p)<\varepsilon<1\), then there exists
\(\eta(p,\varepsilon)>0\) such that
\[
  \mathbb P(\rho_n(\varepsilon,p)>\eta(p,\varepsilon))\to1.
\]
\end{proposition}

\begin{proof}
Take \(d=K(p)\).  By Proposition~\ref{prop:fixed-degree-app},
\[
  M_{d,n}(p)\to F(p)<\varepsilon
\]
in probability, so \(T_{\mathrm{CG},n}(\varepsilon,p)\le d\) with probability
tending to one.  Proposition~\ref{prop:rcd-lower-app} gives
\[
  T_{\mathrm{RCD},n}(\varepsilon,p;I)/n>a
\]
with probability tending to one for some \(a>0\).  On the intersection,
\(\rho_n>a/d\).
\end{proof}

\begin{lemma}[Random hyperplane contraction]
\label{lem:random-hyperplane-contraction-app}
Let \(\mathcal H\) be a separable Hilbert space, and let
\(G_0,G_1,\ldots\) be iid nonzero random vectors whose law has full support in
\(\mathcal H\).  For \(g\ne0\), let \(P_{g^\perp}\) be the orthogonal
projection onto \(g^\perp\).  Then, for every fixed \(h\in\mathcal H\),
\[
  P_{G_{t-1}^\perp}\cdots P_{G_0^\perp}h\to0
  \qquad\text{strongly almost surely}.
\]
\end{lemma}

\begin{proof}
We use the projection-product theorem for random products of orthogonal
projections in Hilbert space: for iid random closed subspaces \(M_t\), the
product \(P_{M_{t-1}}\cdots P_{M_0}\) converges strongly almost surely to the
orthogonal projection onto \(\cap_{t\ge0}M_t\) \cite{AmemiyaAndo1965}.
Applying this theorem with \(M_t=G_t^\perp\) gives
\[
  P_{G_{t-1}^\perp}\cdots P_{G_0^\perp}h
  \to
  P_{\cap_{t\ge0}G_t^\perp}h
  \qquad\text{strongly almost surely}.
\]

It remains only to identify the intersection.  Since the law of \(G_0\) has
full support and \(\mathcal H\) is separable, every ball in a countable base
has positive probability.  By independence and Borel--Cantelli, each such ball
is hit infinitely often by the sequence \(\{G_t:t\ge0\}\).  Hence
\(\{G_t:t\ge0\}\) is almost surely dense in \(\mathcal H\), and therefore its
closed linear span is all of \(\mathcal H\).  Thus
\[
  \bigcap_{t\ge0}G_t^\perp=\{0\}
  \qquad\text{almost surely}.
\]
The projection-product limit is consequently \(P_{\{0\}}h=0\), which proves
the claim.
\end{proof}

\begin{lemma}[Summable-spectrum RCD contraction]
\label{lem:summable-rcd-contraction-app}
Fix \(p>1\).  For every \(\varepsilon,\delta\in(0,1)\), there exists a
finite integer \(L=L(p,\varepsilon,\delta)\) such that
\[
  \liminf_{n\to\infty}
  \mathbb P\!\left(
  T_{\mathrm{RCD},n}(\varepsilon,p;I)\le L
  \right)
  \ge 1-\delta .
\]
\end{lemma}

\begin{proof}
Let
\[
  \mathcal H_p
  =
  \left\{
  v=(v_j)_{j\ge1}:
  \|v\|_{\mathcal H_p}^2
  :=
  \sum_{j\ge1}j^{-p}|v_j|^2<\infty
  \right\}.
\]
Since \(p>1\),
\[
  H_p:=\sum_{j\ge1}j^{-p}<\infty,
\]
so a standard complex Gaussian sequence belongs to \(\mathcal H_p\) almost
surely.

For \(1\le i\le n\), define the rescaled Haar row
\[
  g_{i,n}=\sqrt n\,U_n^*\mathbf e_i
\]
and its weighted norm
\[
  D_{i,n}
  =
  \sum_{j=1}^n j^{-p}|g_{i,n,j}|^2.
\]
Since
\[
  A_n(i,i)
  =
  \mathbf e_i^*U_n\Lambda_nU_n^*\mathbf e_i
  =
  \frac1nD_{i,n},
\]
the RCD sampling probability is
\[
  \mathbb P(i_t=i\mid A_n)
  =
  \frac{A_n(i,i)}{\operatorname{tr}(A_n)}
  =
  \frac{D_{i,n}}{nH_{n,p}},
  \qquad
  H_{n,p}=\sum_{j=1}^n j^{-p}.
\]

The weighted empirical row law
\[
  \sum_{i=1}^n
  \frac{D_{i,n}}{nH_{n,p}}\delta_{g_{i,n}}
\]
converges weakly, in probability, to the size-biased Gaussian law
\[
  Q(dg)
  =
  \frac{\|g\|_{\mathcal H_p}^2}{H_p}\Gamma(dg),
\]
where \(\Gamma\) is the law of a standard complex Gaussian sequence in
\(\mathcal H_p\).  The finite-coordinate marginals follow from the Gaussian
representation of Haar rows, while the tail tightness follows from
\(\sum_{j\ge1}j^{-p}<\infty\).

Now write the spectral error as
\[
  \beta_{t,n}=\sqrt n\,U_n^*(x_t-z_n).
\]
For every fixed number of RCD steps, the process \((\beta_{t,n})\) converges
to the Hilbert-space process
\[
  B_{t+1}
  =
  B_t
  -
  \frac{\langle G_t,B_t\rangle_{\mathcal H_p}}
       {\|G_t\|_{\mathcal H_p}^2}
  G_t,
  \qquad
  G_t\stackrel{\mathrm{iid}}{\sim}Q,
\]
with \(B_0\sim\Gamma\).  Equivalently,
\[
  B_{t+1}=P_{G_t^\perp}B_t.
\]

The law \(Q\) has full support in \(\mathcal H_p\).  Therefore, by
Lemma~\ref{lem:random-hyperplane-contraction-app},
\[
  \|B_t\|_{\mathcal H_p}\to0
  \qquad\text{almost surely}.
\]
Equivalently,
\[
  \frac{\|B_t\|_{\mathcal H_p}^2}{\|B_0\|_{\mathcal H_p}^2}\to0
  \qquad\text{almost surely}.
\]

Thus, for the given \(\varepsilon,\delta\), one can choose a fixed integer
\(L\) such that
\[
  \mathbb P\!\left(
  \|B_L\|_{\mathcal H_p}^2
  \le
  \varepsilon\|B_0\|_{\mathcal H_p}^2
  \right)
  \ge 1-\delta .
\]
The fixed-step convergence of \(\beta_{t,n}\) to \(B_t\) then implies
\[
  \liminf_{n\to\infty}
  \mathbb P\!\left(
  T_{\mathrm{RCD},n}(\varepsilon,p;I)\le L
  \right)
  \ge 1-\delta .
\]
\end{proof}

\begin{proposition}[The summable row]
\label{prop:pgt1-app}
If \(p>1\), then
\[
  T_{\mathrm{RCD},n}(\varepsilon,p;I)=O_{\mathbb P}(1),
  \qquad
  \rho_n(\varepsilon,p)=O_{\mathbb P}(n^{-1}).
\]
\end{proposition}

\begin{proof}
By Lemma~\ref{lem:summable-rcd-contraction-app}, for every \(\delta>0\) there
exists a fixed \(L<\infty\) such that
\[
  \liminf_{n\to\infty}
  \mathbb P\!\left(
  T_{\mathrm{RCD},n}(\varepsilon,p;I)\le L
  \right)
  \ge 1-\delta .
\]
This is exactly
\[
  T_{\mathrm{RCD},n}(\varepsilon,p;I)=O_{\mathbb P}(1).
\]
Since \(T_{\mathrm{CG},n}(\varepsilon,p)\ge1\),
\[
  \rho_n(\varepsilon,p)
  =
  \frac{T_{\mathrm{RCD},n}(\varepsilon,p;I)/n}
       {T_{\mathrm{CG},n}(\varepsilon,p)}
  \le
  \frac{T_{\mathrm{RCD},n}(\varepsilon,p;I)}{n}
  =
  O_{\mathbb P}(n^{-1}).
\]
\end{proof}

\subsection{The logarithmic row \(p=1\)}

\begin{lemma}[Random Gram consistency under small leverage]
\label{lem:random-gram-consistency-app}
Let \(a_{m,n}\in\mathbb C^r\) be deterministic vectors satisfying
\[
  \sum_m a_{m,n}a_{m,n}^*=I_r,
  \qquad
  \max_m\|a_{m,n}\|_2^2\le \tau_n.
\]
Let \(E_m\) be independent mean-one exponential random variables.  If
\[
  \tau_n\log r\to0,
\]
then
\[
  \left\|
  \sum_m(E_m-1)a_{m,n}a_{m,n}^*
  \right\|_{\mathrm{op}}
  \xrightarrow{\mathbb P}0.
\]
\end{lemma}

\begin{proof}
Set
\[
  X_n=\sum_m E_m a_{m,n}a_{m,n}^*.
\]
Then
\[
  \mathbb E X_n=\sum_m a_{m,n}a_{m,n}^*=I_r.
\]
The exponential matrix Chernoff bound for sums of independent positive
semidefinite rank-one matrices gives the following estimate: for every fixed
\(0<t<1\), there is an absolute constant \(c>0\) such that
\[
  \mathbb P\!\left(\|X_n-I_r\|_{\mathrm{op}}>t\right)
  \le
  2r\exp\!\left(-\frac{c t^2}{\tau_n}\right).
\]
Since \(\tau_n\log r\to0\), the right-hand side tends to zero for every fixed
\(t>0\).  Therefore
\[
  \|X_n-I_r\|_{\mathrm{op}}\xrightarrow{\mathbb P}0.
\]
Because
\[
  X_n-I_r
  =
  \sum_m(E_m-1)a_{m,n}a_{m,n}^*,
\]
the desired conclusion follows.
\end{proof}

\begin{proposition}[CG non-stopping at \(p=1\)]
\label{prop:p1-cg-app}
Let \(a_n\to\infty\) be deterministic with
\[
  \log a_n=o(\log n),
\]
and set
\[
  d_n=\lfloor a_n\rfloor.
\]
Then
\[
  M_{d_n,n}(1)\xrightarrow{\mathbb P}1,
  \qquad
  \mathbb P(T_{\mathrm{CG},n}(\varepsilon,1)\le a_n)\to0.
\]
\end{proposition}

\begin{proof}
Set
\[
  J_n=\lceil d_n^4\rceil,
  \qquad
  I_n=\{J_n,\ldots,n\}.
\]
Since \(\log d_n=o(\log n)\), we have
\[
  \sum_{m\in I_n}\frac1m
  =
  \log n-\log J_n+O(1)
  =
  (1-o(1))\log n,
  \qquad
  d_n^2/J_n\to0.
\]

Consider the deterministic tail inner product
\[
  \langle f,g\rangle_n
  =
  \sum_{m\in I_n}\frac{f(1/m)g(1/m)}m.
\]
We project the constant function \(-1\) onto the span of the monomials
\[
  \xi,\xi^2,\ldots,\xi^{d_n}.
\]
The relevant Gram matrix and linear term are
\[
  G_{kl}=\sum_{m\in I_n}m^{-k-l-1},
  \qquad
  b_k=\sum_{m\in I_n}m^{-k-1},
  \qquad
  1\le k,l\le d_n.
\]
Since \(d_n^2/J_n\to0\), these are asymptotic to a diagonally rescaled shifted
Hilbert system:
\[
  G=R(\mathsf H+o(1))R,
  \qquad
  b=R(u+o(1)),
\]
where
\[
  R=\operatorname{diag}(J_n^{-1},\ldots,J_n^{-d_n}),
  \qquad
  \mathsf H_{kl}=\frac1{k+l},
  \qquad
  u_k=\frac1k.
\]
The shifted Hilbert-matrix identity gives
\[
  u^\top\mathsf H^{-1}u
  =
  2\sum_{\ell=1}^{d_n}\frac1\ell
  =
  O(\log d_n).
\]
Therefore the projection energy is \(O(\log d_n)\), and hence
\[
  \inf_{\deg q\le d_n,\ q(0)=1}
  \sum_{m\in I_n}\frac{|q(1/m)|^2}{m}
  \ge
  \sum_{m\in I_n}\frac1m-O(\log d_n)
  =
  (1-o(1))\log n.
  \tag{A.9}
\]

It remains to pass from deterministic weights to exponential weights.  Define
\[
  \widetilde v_m=(1,m^{-1},\ldots,m^{-d_n})^\top,
  \qquad
  \widetilde G
  =
  \sum_{m\in I_n}m^{-1}\widetilde v_m\widetilde v_m^\top.
\]
The Christoffel leverage bound in this tail window gives
\[
  \ell_m
  :=
  m^{-1}\widetilde v_m^\top\widetilde G^{-1}\widetilde v_m
  \le
  C\frac{d_n^2}{J_n}
  =
  o(1).
\]
Set
\[
  a_{m,n}
  =
  \widetilde G^{-1/2}m^{-1/2}\widetilde v_m.
\]
Then
\[
  \sum_{m\in I_n}a_{m,n}a_{m,n}^\top=I_{d_n+1},
  \qquad
  \max_{m\in I_n}\|a_{m,n}\|_2^2=o(1).
\]
Since \(d_n^2/J_n=d_n^{-2}(1+o(1))\), the leverage bound satisfies the
hypothesis of Lemma~\ref{lem:random-gram-consistency-app}.  Hence
\[
  \left\|
  \sum_{m\in I_n}(E_m-1)
  \widetilde G^{-1/2}m^{-1}
  \widetilde v_m\widetilde v_m^\top
  \widetilde G^{-1/2}
  \right\|_{\mathrm{op}}
  \xrightarrow{\mathbb P}0.
\]
Thus the random weighted quadratic form is relatively consistent with the
deterministic one on the polynomial space, and (A.9) implies
\[
  \inf_{\deg q\le d_n,\ q(0)=1}
  \sum_{m\in I_n}\frac{E_m|q(1/m)|^2}{m}
  \ge
  (1-o_{\mathbb P}(1))\log n.
\]
Also,
\[
  \sum_{m=1}^n\frac{E_m}{m}
  =
  (1+o_{\mathbb P}(1))\log n.
\]
Therefore
\[
  M_{d_n,n}(1)\ge1-o_{\mathbb P}(1).
\]
The reverse inequality follows from the feasible polynomial \(q\equiv1\), so
\[
  M_{d_n,n}(1)\xrightarrow{\mathbb P}1.
\]
Since \(\varepsilon<1\),
\[
  \mathbb P(T_{\mathrm{CG},n}(\varepsilon,1)\le d_n)
  =
  \mathbb P(M_{d_n,n}(1)\le\varepsilon)
  \to0.
\]
Finally, \(T_{\mathrm{CG},n}\) is integer-valued and
\(d_n=\lfloor a_n\rfloor\), so
\[
  \{T_{\mathrm{CG},n}(\varepsilon,1)\le a_n\}
  =
  \{T_{\mathrm{CG},n}(\varepsilon,1)\le d_n\}.
\]
This proves the proposition.
\end{proof}

\begin{proposition}[Rate at \(p=1\)]
\label{prop:p1-rate-app}
For every deterministic \(a_n\to\infty\) satisfying
\[
  \log a_n=o(\log n),
\]
one has
\[
  \rho_n(\varepsilon,1)
  =
  O_{\mathbb P}\!\left(\frac{\log n}{a_n}\right).
\]
\end{proposition}

\begin{proof}
Use (A.5), Proposition~\ref{prop:p1-cg-app}, and the event inclusion
\[
  \left\{\rho_n>\frac{B\log n}{a_n}\right\}
  \subset
  \left\{T_{\mathrm{RCD},n}/n>B\log n\right\}
  \cup
  \left\{T_{\mathrm{CG},n}\le a_n\right\}.
\]
First take \(B\) large, then \(n\to\infty\).
\end{proof}

\subsection{Critical boundary}

At \(\varepsilon=F(p)\), the active floor alone is not decisive.  The proof
uses the first inactive Schur-complement correction to the CG polynomial
minimum.

\begin{lemma}[Critical first-inactive Schur asymptotics]
\label{lem:critical-schur-app}
The following first-inactive Schur asymptotics hold.

\begin{enumerate}[label=(\roman*)]
\item If \(K=1\) and \(1/5<p<1/3\), then
\[
  F(p)=\frac{p^2}{(1-2p)^2},
  \qquad
  M_{1,n}(p)=1-\frac{S_{2,n}^2}{S_{1,n}S_{3,n}},
\]
and
\[
  M_{2,n}(p)=M_{1,n}(p)-J_{2,n},
\]
where
\[
  J_{2,n}
  =
  \frac{(S_{3,n}^2-S_{2,n}S_{4,n})^2}
       {S_{1,n}S_{3,n}(S_{5,n}S_{3,n}-S_{4,n}^2)}.
\]
Moreover,
\[
  M_{1,n}(p)-F(p)=O_{\mathbb P}(n^{3p-1}).
\]
If \(1/5<p<1/4\), then
\[
  n^{5p-1}J_{2,n}\Rightarrow L_{2,p},
  \qquad
  L_{2,p}>0\quad\text{a.s.}
\]
If \(p=1/4\), then
\[
  n^{1/4}(\log n)^{-2}J_{2,n}\Rightarrow L_{2,1/4},
  \qquad
  L_{2,1/4}>0\quad\text{a.s.}
\]
If \(1/4<p<1/3\), then
\[
  n^{1-3p}(M_{1,n}(p)-F(p))
  \Rightarrow
  c_pZ_p,
\]
and
\[
  n^{1-3p}J_{2,n}
  \Rightarrow
  c_p\frac{\mathcal A_4^2}{\mathcal A_5},
\]
where \(c_p>0\),
\[
  \mathcal A_4=\sum_{j=1}^{\infty}j^{-4p}E_j,
  \qquad
  \mathcal A_5=\sum_{j=1}^{\infty}j^{-5p}E_j,
\]
and
\[
  Z_p
  =
  \zeta_{\mathrm R}(3p)
  +
  \sum_{j=1}^{\infty}j^{-3p}(E_j-1).
\]

\item Fix \(K\ge2\) and
\[
  \frac1{2K+3}<p<\frac1{2K+1}.
\]
Let \(r_{K,n}\) be the degree-\(K\) minimizer for \(\mu_n\), and let
\(u_{K+1,n}\) be the component of \(y^{K+1}\) orthogonal to
\[
  \operatorname{span}\{y,\ldots,y^K\}
\]
in \(L^2(\mu_n)\).  Then
\[
  M_{K+1,n}(p)=M_{K,n}(p)-J_{K+1,n},
  \qquad
  J_{K+1,n}=\frac{\ell_{K+1,n}^2}{s_{K+1,n}},
\]
where
\[
  \ell_{K+1,n}
  =
  \int r_{K,n}(y)u_{K+1,n}(y)\,\mu_n(dy),
  \qquad
  s_{K+1,n}
  =
  \int u_{K+1,n}(y)^2\,\mu_n(dy).
\]
The active-layer error satisfies
\[
  M_{K,n}(p)-F(p)
  =
  O_{\mathbb P}\!\left(n^{(2K+1)p-1}\right).
\]
If
\[
  \frac1{2K+3}<p<\frac1{2K+2},
\]
then
\[
  n^{(2K+3)p-1}J_{K+1,n}\Rightarrow L_{K,p},
  \qquad
  L_{K,p}>0\quad\text{a.s.}
\]
If \(p=1/(2K+2)\), then
\[
  J_{K+1,n}
  =
  \frac{L_{K,p}+o_{\mathbb P}(1)}{\log n},
  \qquad
  L_{K,p}>0.
\]

\item If
\[
  \frac1{2K+2}<p<\frac1{2K+1},
  \qquad
  \alpha=2K+1,
\]
then
\[
  n^{1-\alpha p}(M_{K,n}(p)-F(p))
  \Rightarrow
  C_{K,p}Z_{\alpha,p},
\]
and
\[
  n^{1-\alpha p}J_{K+1,n}
  \Rightarrow
  C_{K,p}
  \frac{\mathcal A_{\alpha+1}^2}{\mathcal A_{\alpha+2}},
\]
where \(C_{K,p}>0\),
\[
  \mathcal A_{\alpha+1}
  =
  \sum_{j=1}^{\infty}j^{-(\alpha+1)p}E_j,
  \qquad
  \mathcal A_{\alpha+2}
  =
  \sum_{j=1}^{\infty}j^{-(\alpha+2)p}E_j,
\]
and
\[
  Z_{\alpha,p}
  =
  \zeta_{\mathrm R}(\alpha p)
  +
  \sum_{j=1}^{\infty}j^{-\alpha p}(E_j-1).
\]

\item If \(p_K=1/(2K+3)\), \(K\ge1\), then
\[
  J_{K+1,n}
  =
  \frac{c_K}{\log n}(1+o_{\mathbb P}(1)),
  \qquad
  c_K>0,
\]
and
\[
  M_{K,n}(p_K)-F(p_K)
  =
  O_{\mathbb P}\!\left(n^{-2/(2K+3)}\right).
\]
\end{enumerate}
\end{lemma}

\begin{proof}
The proof is a direct Schur-complement expansion using the moment asymptotics
of
\[
  S_{r,n}=\sum_{j=1}^nY_{j,n}^rE_j.
\]
The needed regimes are
\[
  S_{r,n}
  =
  \frac{n}{1-rp}(1+o_{\mathbb P}(1)),
  \qquad rp<1,
\]
\[
  S_{r,n}
  =
  n\log n\,(1+o_{\mathbb P}(1)),
  \qquad rp=1,
\]
and
\[
  n^{-rp}S_{r,n}
  \to
  \sum_{j=1}^{\infty}j^{-rp}E_j,
  \qquad rp>1.
\]
In the upper subbands one also uses the Euler--Maclaurin fluctuation
\[
  n^{1-rp}
  \left(
  \frac{S_{r,n}}{n}-\frac1{1-rp}
  \right)
  \Rightarrow
  \zeta_{\mathrm R}(rp)
  +
  \sum_{j=1}^{\infty}j^{-rp}(E_j-1),
\]
valid in the \(L^2\)-range \(2rp>1\).

For \(K=1\), the degree-one formula and the degree-two Schur complement give
the displayed expression for \(J_{2,n}\).  Substituting the three moment
regimes above yields the lower-subband, logarithmic, and upper-subband limits.

For \(K\ge2\), the degree-\((K+1)\) improvement is the Schur complement
\[
  J_{K+1,n}=\frac{\ell_{K+1,n}^2}{s_{K+1,n}}.
\]
The numerator is determined by the coupling of the active minimizer with the
first inactive direction \(y^{K+1}\), while the denominator is the squared
\(L^2(\mu_n)\)-distance of \(y^{K+1}\) from the active span.  Applying the same
moment asymptotics gives the stated lower-subband, logarithmic, upper-subband,
and endpoint estimates.  The constants \(L_{2,p}\), \(L_{K,p}\), and \(c_K\)
are strictly positive because the corresponding finite moment Gram Schur
complements are strictly positive.
\end{proof}

\begin{proposition}[Critical boundary]
\label{prop:critical-app}
If \(0<p<1/3\) and \(\varepsilon=F(p)\), then there exist
\(\eta(p)>0\) and \(c(p)>0\) such that
\[
  \liminf_{n\to\infty}
  \mathbb P(\rho_n(F(p),p)>\eta(p))\ge c(p).
\]
Moreover, if
\[
  p=\frac1{2K+3}
  \quad\text{or}\quad
  \frac1{2K+3}<p\le\frac1{2K+2},
  \qquad K\ge1,
\]
then the probability tends to one.
\end{proposition}

\begin{proof}
We first analyze the CG side and then combine it with the RCD lower gate.

\medskip
\noindent\textbf{Step 1: the band \(K=1\).}
Assume
\[
  \frac15<p<\frac13.
\]
By Lemma~\ref{lem:critical-schur-app},
\[
  M_{2,n}(p)=M_{1,n}(p)-J_{2,n}.
\]

If \(1/5<p<1/4\), then
\[
  M_{1,n}(p)-F(p)=O_{\mathbb P}(n^{3p-1}),
  \qquad
  n^{5p-1}J_{2,n}\Rightarrow L_{2,p}>0.
\]
Since
\[
  1-5p>3p-1,
\]
the Schur improvement dominates the active-layer error.  Hence
\[
  \mathbb P(M_{2,n}(p)<F(p))\to1,
\]
and therefore
\[
  \mathbb P(T_{\mathrm{CG},n}(F(p),p)\le2)\to1.
\]

At \(p=1/4\), Lemma~\ref{lem:critical-schur-app} gives
\[
  M_{1,n}(p)-F(p)=O_{\mathbb P}(n^{-1/4}),
\]
and
\[
  n^{1/4}(\log n)^{-2}J_{2,n}\Rightarrow L_{2,1/4}>0.
\]
Thus \(J_{2,n}\) again dominates the active error, and
\[
  \mathbb P(T_{\mathrm{CG},n}(F(p),p)\le2)\to1.
\]

If \(1/4<p<1/3\), Lemma~\ref{lem:critical-schur-app} gives
\[
  n^{1-3p}(F(p)-M_{2,n}(p))
  \Rightarrow
  c_p\left(
  \frac{\mathcal A_4^2}{\mathcal A_5}-Z_p
  \right).
\]
The limiting random variable is positive with positive probability: indeed,
\(\mathcal A_4^2/\mathcal A_5>0\) almost surely, while
\[
  \mathbb E Z_p=\zeta_{\mathrm R}(3p)<0
  \qquad (0<3p<1).
\]
Hence
\[
  \liminf_{n\to\infty}
  \mathbb P(T_{\mathrm{CG},n}(F(p),p)\le2)>0.
\]

\medskip
\noindent\textbf{Step 2: interior bands with \(K\ge2\).}
Fix \(K\ge2\) and assume
\[
  \frac1{2K+3}<p<\frac1{2K+1}.
\]
By Lemma~\ref{lem:critical-schur-app},
\[
  M_{K+1,n}(p)=M_{K,n}(p)-J_{K+1,n},
\]
and
\[
  M_{K,n}(p)-F(p)
  =
  O_{\mathbb P}\!\left(n^{(2K+1)p-1}\right).
\]

If
\[
  \frac1{2K+3}<p<\frac1{2K+2},
\]
then
\[
  n^{(2K+3)p-1}J_{K+1,n}\Rightarrow L_{K,p}>0.
\]
Since
\[
  1-(2K+3)p>(2K+1)p-1,
\]
the Schur improvement dominates the active-layer error.  Hence
\[
  \mathbb P(T_{\mathrm{CG},n}(F(p),p)\le K+1)\to1.
\]

At
\[
  p=\frac1{2K+2},
\]
Lemma~\ref{lem:critical-schur-app} gives
\[
  J_{K+1,n}
  =
  \frac{L_{K,p}+o_{\mathbb P}(1)}{\log n},
  \qquad L_{K,p}>0,
\]
whereas
\[
  M_{K,n}(p)-F(p)
  =
  O_{\mathbb P}\!\left(n^{-1/(2K+2)}\right).
\]
Again the Schur improvement dominates, so
\[
  \mathbb P(T_{\mathrm{CG},n}(F(p),p)\le K+1)\to1.
\]

If
\[
  \frac1{2K+2}<p<\frac1{2K+1},
\]
let
\[
  \alpha=2K+1.
\]
Lemma~\ref{lem:critical-schur-app} gives
\[
  n^{1-\alpha p}(F(p)-M_{K+1,n}(p))
  \Rightarrow
  C_{K,p}
  \left(
  \frac{\mathcal A_{\alpha+1}^2}{\mathcal A_{\alpha+2}}
  -
  Z_{\alpha,p}
  \right).
\]
The limiting variable is positive with positive probability because
\[
  \frac{\mathcal A_{\alpha+1}^2}{\mathcal A_{\alpha+2}}>0
  \quad\text{a.s.},
  \qquad
  \mathbb E Z_{\alpha,p}
  =
  \zeta_{\mathrm R}(\alpha p)<0
  \quad (0<\alpha p<1).
\]
Thus
\[
  \liminf_{n\to\infty}
  \mathbb P(T_{\mathrm{CG},n}(F(p),p)\le K+1)>0.
\]

\medskip
\noindent\textbf{Step 3: odd-reciprocal endpoints.}
Let
\[
  p_K=\frac1{2K+3},
  \qquad K\ge1.
\]
By Lemma~\ref{lem:critical-schur-app},
\[
  J_{K+1,n}
  =
  \frac{c_K}{\log n}(1+o_{\mathbb P}(1)),
  \qquad c_K>0,
\]
and
\[
  M_{K,n}(p_K)-F(p_K)
  =
  O_{\mathbb P}\!\left(n^{-2/(2K+3)}\right).
\]
Since \(1/\log n\) dominates \(n^{-2/(2K+3)}\),
\[
  \mathbb P(T_{\mathrm{CG},n}(F(p_K),p_K)\le K+1)\to1.
\]

\medskip
\noindent\textbf{Step 4: combination with the RCD lower gate.}
For \(0<p<1/3\), one has \(F(p)<1\).  Hence
Proposition~\ref{prop:rcd-lower-app} applies at \(\varepsilon=F(p)\): there
exists \(a(p)>0\) such that
\[
  \mathbb P\!\left(
  T_{\mathrm{RCD},n}(F(p),p;I)/n>a(p)
  \right)\to1.
\]
If the CG stopping event has positive liminf probability, then on the
intersection
\[
  \rho_n(F(p),p)
  \ge
  \frac{a(p)}{K(p)+1}.
\]
This gives constants \(\eta(p)>0\) and \(c(p)>0\) such that
\[
  \liminf_{n\to\infty}
  \mathbb P(\rho_n(F(p),p)>\eta(p))\ge c(p).
\]
On the endpoint/lower-subband regimes, the CG stopping event has probability
tending to one, so the same argument gives
\[
  \mathbb P(\rho_n(F(p),p)>\eta(p))\to1
\]
for some \(\eta(p)>0\).
\end{proof}

\subsection{Growing-degree rate bounds for \(0<p<1\)}

The zero-ratio rows with \(0<p<1\) use a conservative growing-degree
non-stopping estimate.  Throughout this subsection,
\[
  \langle f,g\rangle_n=\int f(y)g(y)\,\mu_n(dy),
  \qquad
  \|f\|_n^2=\langle f,f\rangle_n .
\]
Let \(K=K(p)\).  Let \(r_{K,n}\) be the degree-\(K\) minimizer in
\(L^2(\mu_n)\) subject to \(r(0)=1\).  Thus
\[
  M_{K,n}(p)=\|r_{K,n}\|_n^2.
\]
If \(K\ge1\), then
\[
  \langle r_{K,n},y^a\rangle_n=0,\qquad 1\le a\le K,
\]
and we set
\[
  V_{K,n}=\operatorname{span}\{y,\ldots,y^K\}.
\]
If \(K=0\), then
\[
  r_{0,n}\equiv1,
  \qquad
  V_{0,n}=\{0\}.
\]

\begin{lemma}[Growing-degree Schur bound]
\label{lem:growing-schur-app}
Let \(d=d_n\to\infty\) be deterministic and satisfy
\[
  d\log(1+d)=o(\log n).
\]
Then
\[
  M_{K,n}(p)-M_{d,n}(p)=o_{\mathbb P}(1).
\]
More precisely:

\begin{enumerate}[label=(\roman*)]
\item If \(p\) is not an odd reciprocal endpoint, then there are constants
\(A,C<\infty\) and \(\beta=\beta(p)>0\) such that
\[
  0\le M_{K,n}(p)-M_{d,n}(p)
  \le
  O_{\mathbb P}\!\left(
    e^{Cd\log(1+d)}d^A(\log n)^A n^{-\beta}
  \right).
  \tag{A.10}
\]

\item If \(p=p_K=1/(2K+3)\), then there are constants \(A,C,c>0\) such that
\[
\begin{aligned}
  0\le M_{K,n}(p_K)-M_{d,n}(p_K)
  &\le
  O_{\mathbb P}\!\left(
    \frac{1}{\log n-Cd\log(1+d)}
  \right)  \\
  &\quad+
  O_{\mathbb P}\!\left(
    e^{Cd\log(1+d)}d^A(\log n)^A n^{-c}
  \right).
\end{aligned}
  \tag{A.11}
\]
\end{enumerate}
\end{lemma}

\begin{proof}
The Schur-complement form of the improvement from degree \(K\) to degree \(d\)
is
\[
  M_{K,n}-M_{d,n}
  =
  \sup_{0\ne s\in\operatorname{span}\{y,\ldots,y^d\}}
  \frac{\langle r_{K,n},s\rangle_n^2}{\|s\|_n^2}.
  \tag{A.12}
\]
Write
\[
  s(y)=\sum_{k=1}^d c_k y^k,
  \qquad
  a_k=n^{pk}c_k.
\]
At the first \(d\) spectral nodes \(Y_{j,n}=(n/j)^p\),
\[
  s(Y_{j,n})=\sum_{k=1}^d a_k j^{-pk}.
\]
The Vandermonde matrix \(B_d=(j^{-pk})_{1\le j,k\le d}\) satisfies
\[
  \|B_d^{-1}\|\le \exp\{Cd\log(1+d)\}.
\]
Indeed, the node separation
\[
  |i^{-p}-j^{-p}|\ge c_p d^{-p-1}|i-j|
\]
and the Lagrange interpolation formula give this inverse bound.  Since
\(\min_{1\le j\le d}E_j\ge d^{-3}\) with probability tending to one and
\(S_{1,n}/n=O_{\mathbb P}(1)\) with inverse tight, we get
\[
  \|s\|_n^2
  \ge
  n^{p-1}e^{-Cd\log(1+d)}
  \sum_{k=1}^d a_k^2
  \tag{A.13}
\]
with high probability.

For the numerator, the orthogonality of \(r_{K,n}\) to
\(y,\ldots,y^K\) gives
\[
  \langle r_{K,n},s\rangle_n
  =
  \sum_{k=K+1}^d a_k\gamma_{k,n},
  \qquad
  \gamma_{k,n}=n^{-pk}\langle r_{K,n},y^k\rangle_n .
\]
The coefficients of \(r_{K,n}\) are tight, and for
\(0\le a\le K\), \(K+1\le k\le d\),
\[
  \langle y^a,y^k\rangle_n=\frac{S_{a+k+1,n}}{S_{1,n}}.
\]
A union bound over \(k\le d\), together with Markov estimates for these
moments, gives
\[
  \sum_{k=K+1}^d\gamma_{k,n}^2
  =
  O_{\mathbb P}\!\left(
    d^A(\log n)^A
    \bigl[n^{-2p(K+1)}+n^{2(K+1)p-2}\bigr]
  \right).
  \tag{A.14}
\]
Combining (A.12), (A.13), and (A.14),
\[
  M_{K,n}-M_{d,n}
  \le
  O_{\mathbb P}\!\left(
    e^{Cd\log(1+d)}d^A(\log n)^A
    \bigl[
      n^{1-(2K+3)p}+n^{(2K+1)p-1}
    \bigr]
  \right).
\]
If \(p\) is not an odd reciprocal endpoint, then
\[
  (2K+1)p<1<(2K+3)p,
\]
so the bracket is \(O(n^{-\beta})\) for some \(\beta(p)>0\).  This proves
(A.10).

It remains to explain the endpoint \(p=p_K=1/(2K+3)\).  The first inactive
direction \(y^{K+1}\) is logarithmic and has to be separated from the higher
inactive block.  Using the convention above, define
\[
  u_n=(I-P_{V_{K,n}})y^{K+1},
\]
and
\[
  B_{d,n}=(I-P_{V_{K,n}})
  \operatorname{span}\{y^{K+2},\ldots,y^d\},
\]
where projections are in \(L^2(\mu_n)\).  Since \(r_{K,n}\perp V_{K,n}\),
\[
  M_{K,n}-M_{d,n}
  =
  \bigl\|P_{\operatorname{span}\{u_n\}+B_{d,n}}r_{K,n}\bigr\|_n^2 .
\]
A deterministic two-subspace projection estimate gives
\[
  \bigl\|P_{\operatorname{span}\{u_n\}+B_{d,n}}r_{K,n}\bigr\|_n^2
  \le
  O_{\mathbb P}(1)
  \left(
    \|P_{B_{d,n}}r_{K,n}\|_n^2
    +
    \frac{|\langle r_{K,n},u_n\rangle_n|^2}
         {\operatorname{dist}_n^2(u_n,B_{d,n})}
  \right).
  \tag{A.15}
\]
At the endpoint, the distance term is logarithmic:
\[
  \operatorname{dist}_n^2(u_n,B_{d,n})
  \ge
  c\bigl(\log n-Cd\log(1+d)\bigr)
\]
with high probability.  This is the endpoint Christoffel estimate obtained
by changing variables \(t=\log y\): the problem becomes the projection of a
constant on an interval of length \(p\log n\) against an exponential system,
whose Cauchy Gram matrix loses only \(O(d\log(1+d))\).  The random version
follows from the same small-leverage Gram consistency used in the \(p=1\)
row.

Moreover,
\[
  \langle r_{K,n},u_n\rangle_n
  =
  \langle r_{K,n},y^{K+1}\rangle_n
  =
  O_{\mathbb P}(1),
\]
because the highest active moment involved is still subcritical.  Thus the
first inactive direction contributes the first term in (A.11).  The block
\(B_{d,n}\) is controlled by the same Vandermonde argument as above, now
starting at \(K+2\), which gives the second term in (A.11).  Since
\(d\log(1+d)=o(\log n)\), both terms are \(o_{\mathbb P}(1)\).
\end{proof}

\begin{proposition}[Growing-degree CG non-stopping]
\label{prop:growing-cg-app}
Assume \(0<p<1\), \(\varepsilon<F(p)\), and let \(b_n\to\infty\) satisfy
\[
  b_n\log(1+b_n)=o(\log n).
  \tag{A.16}
\]
Then
\[
  \mathbb P(T_{\mathrm{CG},n}(\varepsilon,p)\le b_n)\to0.
\]
\end{proposition}

\begin{proof}
Set
\[
  d_n=\lfloor b_n\rfloor .
\]
Then \(d_n\to\infty\), \(d_n\log(1+d_n)=o(\log n)\), and since
\(T_{\mathrm{CG},n}\) is integer-valued,
\[
  \{T_{\mathrm{CG},n}(\varepsilon,p)\le b_n\}
  =
  \{T_{\mathrm{CG},n}(\varepsilon,p)\le d_n\}.
\]
By Proposition~\ref{prop:fixed-degree-app},
\[
  M_{K,n}(p)\xrightarrow{\mathbb P}F(p).
\]
By Lemma~\ref{lem:growing-schur-app},
\[
  0\le M_{K,n}(p)-M_{d_n,n}(p)=o_{\mathbb P}(1).
\]
Hence
\[
  M_{d_n,n}(p)\ge F(p)-o_{\mathbb P}(1).
\]
Choose \(\delta=(F(p)-\varepsilon)/2>0\).  Then
\[
  \mathbb P(M_{d_n,n}(p)\le\varepsilon)
  \le
  \mathbb P(M_{d_n,n}(p)<F(p)-\delta)
  \to0.
\]
Using the CG polynomial representation,
\[
  \mathbb P(T_{\mathrm{CG},n}(\varepsilon,p)\le b_n)\to0.
\]
\end{proof}

\begin{proposition}[Rate for \(0<p<1\)]
\label{prop:subunit-rate-app}
In every zero-ratio row with \(0<p<1\), if \(b_n\to\infty\) and
\[
  b_n\log(1+b_n)=o(\log n),
\]
then
\[
  \rho_n(\varepsilon,p)=O_{\mathbb P}(1/b_n).
\]
\end{proposition}

\begin{proof}
Use (A.4), Proposition~\ref{prop:growing-cg-app}, and
\[
  \left\{\rho_n>\frac{B}{b_n}\right\}
  \subset
  \left\{T_{\mathrm{RCD},n}/n>B\right\}
  \cup
  \left\{T_{\mathrm{CG},n}\le b_n\right\}.
\]
First choose \(B\) large, then let \(n\to\infty\).
\end{proof}

\subsection{Completion}

The zero-ratio rows are Proposition~\ref{prop:pgt1-app},
Proposition~\ref{prop:p1-rate-app}, Proposition~\ref{prop:below-floor-app},
and Proposition~\ref{prop:subunit-rate-app}.  The strictly above-floor
positive-ratio row is Proposition~\ref{prop:above-floor-app}.  The critical
boundary is Proposition~\ref{prop:critical-app}.

All statements are fixed-parameter statements.  The displayed rates are
proved conservative upper bounds, not sharp rates.  The proof does not address
matching lower bounds, uniformity in \(p\) or \(K\), moving thresholds,
finite-precision effects, preconditioning, restarted CG, non-Haar eigenvectors,
alternative sketching rules, or solver-level performance comparisons.

\section{Proof details for the QRCP obstruction}\label{app:qrcp-proof}

This appendix records the proof structure behind Theorem~\ref{thm:qrcp-source-scale-obstruction}.  The purpose is to make the counterexample readable as a mathematical construction rather than as a computational search log.

The proof is organized as follows.  We first state the construction.  We then prove five named estimates: pivot order, Parseval identity, coherence bound, selected-block conditioning, and source-scale violation.

\subsection{The selected columns}

We build column vectors $A=\begin{pmatrix}a_1&\cdots&a_n\end{pmatrix}\in\mathbb{R}^{k\times n}$,
and then define $Q=A^T$.

\medskip
Fix $H_*>1$, $B>0$, $e\ge0$, and $K\ge1$.  Choose $\eta>0$ so small that
\[
e^{2\eta}<H_*.
\]
Choose $m$ large and set
\[
k=m+1,
\qquad
\epsilon=e^{-\eta/k},
\qquad
s=(1-\epsilon^2)^{1/2},
\qquad
\alpha=\frac1k.
\]

All smallness requirements below are met by taking $k$ sufficiently large. Define the diagonal scales
\[
d_t=\alpha\epsilon^{t-1},
\qquad 1\le t\le m.
\]

The first $k=m+1$ column labels are the intended QRCP labels:
\[
I_{\mathrm{piv}}=(1,2,\ldots,m,m+1).
\]

Choose signs $\sigma_t\in\{\pm1\}$ and set $\beta=\frac12$. This constant controls the strictly upper-triangular couplings among the first $m$ selected columns. Define $R_\beta\in\mathbb{R}^{m\times m}$ by
\[
(R_\beta)_{t,q}=
\begin{cases}
 d_t, & q=t,\\
 -\beta\sigma_t\sigma_qs d_t, & t<q,\\
 0, & t>q.
\end{cases}
\]
The first $m$ selected columns are
\[
a_q=\begin{pmatrix}
    (R_\beta)_{:,q} \\
    0
\end{pmatrix},
\qquad 1\le q\le m.
\]
In matrix form their first $m$ coordinates are
\[
R_\beta=
\begin{pmatrix}
 d_1 & -\beta\sigma_1\sigma_2s d_1 & \cdots & -\beta\sigma_1\sigma_ms d_1\\
 0   & d_2                          & \cdots & -\beta\sigma_2\sigma_ms d_2\\
 \vdots & \vdots                    & \ddots & \vdots\\
 0 & 0 & \cdots & d_m
\end{pmatrix}.
\]
We now define the final selected column.  Set
\[
\gamma_0=\frac14,
\qquad
\tau_0=\frac1{16}.
\]
The first number controls the size of the first-block component of the final selected column; the second fixes its final residual scale.  Define
\[
r=\tau_0d_m^2,
\qquad
\eta_p=\tau_0^{1/2}d_m,
\]
so that $\eta_p^2=r$.  After the first $m$ selected columns have been chosen, their span is $\mathbb{R}^m\oplus\{0\}$; projecting the final selected column away from this span removes its first-block part and leaves only $\eta_p$.  Thus its squared residual at that moment is $r$.

Define $z\in\mathbb{R}^m$ by
\[
z_t=
\begin{cases}
\gamma_0\sigma_ts d_t, & 1\le t<m,\\
\gamma_0\sigma_m(1-\tau_0)^{1/2}d_m, & t=m.
\end{cases}
\]
The final selected column is
\[
a_{m+1}=\begin{pmatrix}
    z \\
    \eta_p
\end{pmatrix},
\]
Thus the selected block of $A$ is
\[
R_+=
\begin{pmatrix}
R_\beta & z\\
0 & \eta_p
\end{pmatrix}.
\]
Equivalently,
\[
R_+=
\begin{pmatrix}
 d_1 & -\beta\sigma_1\sigma_2s d_1 & \cdots & -\beta\sigma_1\sigma_ms d_1 & \gamma_0\sigma_1s d_1\\
 0   & d_2                          & \cdots & -\beta\sigma_2\sigma_ms d_2 & \gamma_0\sigma_2s d_2\\
 \vdots & \vdots                    & \ddots & \vdots                       & \vdots\\
 0 & 0 & \cdots & d_m & \gamma_0\sigma_m(1-\tau_0)^{1/2}d_m\\
 0 & 0 & \cdots & 0 & \tau_0^{1/2}d_m
\end{pmatrix}.
\]
Since $Q=A^T$, the selected row matrix $M(Q)$ is $R_+^T$ up to the fixed ordering.

\subsection{The full matrix and the identity completion}
\label{subsec:B-completing-identity}

We now display the whole matrix.  The construction has four column blocks:

\[
A=\begin{pmatrix} R_+ & X & F & G \end{pmatrix}\in\mathbb{R}^{k\times n},
\qquad k=m+1.
\]

We use the splitting
\[
\mathbb{R}^k=\mathbb{R}^m\oplus\mathbb{R},
\qquad
I_k=\begin{pmatrix}I_m&0\\0&1\end{pmatrix}.
\]

This block split separates the two residual mechanisms in the construction.  The first $m$ coordinates carry the QRCP chain and the triangular amplification; the last coordinate stores the residual mass that remains after the chain has been selected. Thus, $AA^T=I_k$ is proved by checking the off-diagonal blocks, the upper-left block, and the lower-right scalar separately.

\subsubsection*{The cross block $X$}

The final selected column $\begin{pmatrix} z \\ \eta_p\end{pmatrix}$ creates an off-diagonal block $z\eta_p$.  We cancel it by adding two columns.  Choose
\[
\theta_1=\frac35,
\quad \theta_2=\frac45,
\quad b_1=\frac7{15},
\quad b_2=\frac9{10},
\qquad
\theta_1b_1+\theta_2b_2=1.
\]
Define
\[
X=
\begin{pmatrix}
\theta_1z&\theta_2z\\
-b_1\eta_p&-b_2\eta_p
\end{pmatrix}.
\]
Then
\[
R_+R_+^T
=
\begin{pmatrix}
R_\beta R_\beta^T+zz^T&z\eta_p\\
\eta_pz^T&\eta_p^2
\end{pmatrix}
\]
and
\[
XX^T
=
\begin{pmatrix}
(\theta_1^2+\theta_2^2)zz^T&-z\eta_p\\
-\eta_pz^T&(b_1^2+b_2^2)\eta_p^2
\end{pmatrix}.
\]
So the off-diagonal blocks in $R_+R_+^T+XX^T$ are zero.  Put
\[
\kappa_z=1+\theta_1^2+\theta_2^2,
\qquad
\kappa_\eta=1+b_1^2+b_2^2.
\]

\subsubsection*{The pure frame completion block $F$}

After $R_+$ and $X$ have been placed, the upper-left block still missing from $I_m$ is
\[
M_{\mathrm{rem}}
=
I_m-R_\beta R_\beta^T-\kappa_z zz^T.
\]
For $m$ sufficiently large, $M_{\mathrm{rem}}$ is positive definite.

Decompose this missing upper-left block into small rank-one atoms:
\[
M_{\mathrm{rem}}
=
\sum_{h=1}^{N_{\mathrm{frame}}}W_hu_hu_h^T,
\qquad
0<W_h<d_m^2,
\qquad
\|{u_h}\|_2=1.
\]
The atom bound is deliberately at the same scale as the smallest selected-chain residual.  It ensures that no frame column can overtake an intended QRCP pivot.  Such a decomposition is obtained by taking a spectral decomposition of $M_{\mathrm{rem}}$ and subdividing each eigenvalue into sufficiently many positive pieces.  The number of atoms satisfies
\[
N_{\mathrm{frame}}
\le
\frac{\mathrm{Tr}(M_{\mathrm{rem}})}{d_m^2}+2m
\le
\frac{m}{d_m^2}+2m.
\]

For each atom add one column supported only in the first $m$ coordinates:
\[
f_h=\begin{pmatrix}
\sqrt{W_h}\,u_h\\
0
\end{pmatrix}\in\mathbb{R}^m\oplus\mathbb{R},
\qquad 1\le h\le N_{\mathrm{frame}}.
\]
Equivalently, the whole frame block is the explicit matrix
\[
F=
\begin{pmatrix}
\sqrt{W_1}u_1&\sqrt{W_2}u_2&\cdots&\sqrt{W_{N_{\mathrm{frame}}}}u_{N_{\mathrm{frame}}}\\[1mm]
0&0&\cdots&0
\end{pmatrix}
\in\mathbb{R}^{k\times N_{\mathrm{frame}}}.
\]
Then
\[
FF^T=
\begin{pmatrix}
\displaystyle\sum_{h=1}^{N_{\mathrm{frame}}}W_hu_hu_h^T&0\\
0&0
\end{pmatrix}
=
\begin{pmatrix}
M_{\mathrm{rem}}&0\\
0&0
\end{pmatrix}.
\]
Thus $F$ completes exactly the missing upper-left block, produces no off-diagonal block, and contributes no mass to the last coordinate.

The same small-atom bound controls QRCP residuals.  At a selected prefix $S_\ell$ with $\ell<m$, the residual score of $f_h$ is at most its full squared norm:
\[
\|{f_h}\|_2^2=W_h<d_m^2\le d_{\ell+1}^2.
\]
At the last prefix $S_m$, the selected span is $\mathbb{R}^m\oplus\{0\}$, so $f_h$ has residual score zero.  Hence the pure frame columns never compete with the prescribed pivots.

\subsubsection*{The residual-only block $G$}

After the blocks $R_+$, $X$, and $F$ have been placed, the upper-left block has already become $I_m$ and the off-diagonal block is zero. Since $F$ has no last-coordinate component, the only lower-right contribution so far is the selected-and-cross contribution $\kappa_\eta r$.  The amount still missing is therefore
\[
\mu_{\mathrm{res}}
=1-\kappa_\eta r.
\]
For $m$ sufficiently large, $r=\tau_0d_m^2$ is small and $\mu_{\mathrm{res}}>0$.

Split the remaining scalar into small pieces
\[
\mu_{\mathrm{res}}=e_1+\cdots+e_{N_{\mathrm{res}}},
\qquad
0<e_j<\frac r2,
\]
choosing the $e_j$ also to avoid the finitely many residual equalities if a fixed tie-breaking convention is desired.  This can be done with
\[
N_{\mathrm{res}}
\le
\frac{2}{r}+2.
\]
For each piece add one column supported only in the last coordinate:
\[
g_j=\begin{pmatrix}
0\\
\sqrt{e_j}
\end{pmatrix}\in\mathbb{R}^m\oplus\mathbb{R},
\qquad 1\le j\le N_{\mathrm{res}}.
\]
Equivalently,
\[
G=
\begin{pmatrix}
0_{m\times N_{\mathrm{res}}}\\[1mm]
\sqrt{e_1}\ \sqrt{e_2}\ \cdots\ \sqrt{e_{N_{\mathrm{res}}}}
\end{pmatrix}
\in\mathbb{R}^{k\times N_{\mathrm{res}}}.
\]
Therefore
\[
GG^T=
\begin{pmatrix}
0_{m\times m}&0\\
0&\mu_{\mathrm{res}}
\end{pmatrix}.
\]
These columns cannot compete with the final selected column after the first $m$ pivots, because each has squared residual $e_j<r/2<r$.

Putting the four blocks together,
\[
AA^T=R_+R_+^T+XX^T+FF^T+GG^T.
\]
The upper-left block is
\[
R_\beta R_\beta^T+\kappa_z zz^T+M_{\mathrm{rem}}=I_m,
\]
the off-diagonal block is zero, and the lower-right scalar is
\[
\kappa_\eta r+\mu_{\mathrm{res}}=1.
\]
Hence
\[
AA^T=\begin{pmatrix}I_m&0\\0&1\end{pmatrix}=I_k.
\]
% This is the full construction of $A$, and hence of $Q=A^T$.

% The total number of columns of $A$, equivalently rows of $Q$, is
% \[
% n=m+3+N_{\mathrm{frame}}+N_{\mathrm{res}}.
% \]
% Finally, since
% \[
% d_m^2=\alpha^2\epsilon^{2m-2}
% =\frac{1}{k^2}\exp\left(-\frac{2\eta(m-1)}{k}\right),
% \]
% and $m=k-1$, there are constants $0<c_\eta<C_\eta<\infty$, depending only on $\eta$, such that
% \[
% \frac{c_\eta}{k^2}\le d_m^2\le \frac{C_\eta}{k^2}.
% \]
% Consequently
% \[
% N_{\mathrm{frame}}=O(k^3),
% \qquad
% N_{\mathrm{res}}=O(k^2),
% \qquad
% n=O(k^3).
% \]
\subsection{Main lemmas}
\label{sec:B-main-lemmas}

This section isolates the estimates used to prove the main theorem.

\begin{lemma}[Pivot order]
\label{lem:B-pivot-order}
QRCP applied to $A=Q^T$ selects
\[
I_{\mathrm{piv}}=(1,2,\ldots,m,m+1).
\]
\end{lemma}

\begin{proof}

Define $R_j(S)=\|{(I_k-P_S)a_j}\|_2^2$, where $P_S$ is the orthogonal projector onto $\mathrm{span}\{a_i:i\in S\}$. For $0\le \ell\le m$, put $S_\ell=\{1,\ldots,\ell\}$. At prefix $S_\ell$ with $0\le\ell<m$, the intended next label is $\ell+1$ and
\[
R_{\ell+1}(S_\ell)=d_{\ell+1}^2.
\]
For a later selected-chain label $q$ with $\ell+1<q\le m$, the upper-triangular form of $R_\beta$ gives
\[
R_q(S_\ell)=\beta^2s^2\sum_{a=\ell+1}^{q-1}d_a^2+d_q^2.
\]
Using $d_a=\alpha\epsilon^{a-1}$ and $s^2=1-\epsilon^2$, this becomes
\[
R_q(S_\ell)
=d_{\ell+1}^2\left(\beta^2+(1-\beta^2)\epsilon^{2(q-\ell-1)}\right)
<d_{\ell+1}^2,
\]
because $0<\epsilon<1$ and $0<\beta<1$.

The final selected label $m+1$ satisfies, for $0\le\ell<m$,
\[
R_{m+1}(S_\ell)=\gamma_0^2(d_{\ell+1}^2-r)+r.
\]
Since $r=\tau_0d_m^2\le \tau_0d_{\ell+1}^2$,
\[
R_{m+1}(S_\ell)
\le (\gamma_0^2+\tau_0)d_{\ell+1}^2
=\frac18d_{\ell+1}^2
<d_{\ell+1}^2.
\]

The columns in $X$, $F$, and $G$ are all smaller.  The two columns in $X$ have the same form as the final selected column, but with both coefficients strictly smaller than one; the constants above give residual scores strictly below the final selected column at every prefix. For a pure frame column $f_h$, the atom bound gives, for $\ell<m$,
\[
 \|{f_h}\|_2^2=W_h<d_m^2\le d_{\ell+1}^2.
\]
For a residual-only column $g_j$, its residual score is at most $e_j<r/2<d_{\ell+1}^2$.

At $S_m$, the final selected label has residual score
\[
R_{m+1}(S_m)=r,
\]
whereas the columns in $X$ have last-coordinate residual $b_i^2r<r$, each frame column has residual zero, and each residual-only column has residual $e_j<r/2<r$. Thus $m+1$ is selected at the last step.
\end{proof}

\begin{lemma}[Parseval identity]
\label{lem:B-parseval}
The constructed matrix satisfies
\[
AA^T=I_k,
\qquad\text{and hence}\qquad
Q^T Q=I_k.
\]
\end{lemma}

\begin{proof}
Using the block notation from the construction,
\[
A=\begin{pmatrix}R_+&X&F&G\end{pmatrix}.
\]
Thus
\[
AA^T=R_+R_+^T+XX^T+FF^T+GG^T.
\]
The selected block contributes
\[
R_+R_+^T=
\begin{pmatrix}
R_\beta R_\beta^T+zz^T&z\eta_p\\
\eta_pz^T&\eta_p^2
\end{pmatrix}.
\]
The cross block $X$ cancels the off-diagonal entries because $\theta_1b_1+\theta_2b_2=1$. The pure frame block $F$ contributes exactly $M_{\mathrm{rem}}$ in the upper-left corner and contributes zero to both the off-diagonal blocks and the lower-right entry. The block $G$ contributes exactly the remaining lower-right scalar
\[
\mu_{\mathrm{res}}=1-\kappa_\eta r.
\]
Therefore
\[
AA^T=\begin{pmatrix}I_m&0\\0&1\end{pmatrix}=I_k.
\]
\end{proof}

\begin{lemma}[Coherence bound]
\label{lem:B-coherence}
For $m$ sufficiently large, the constructed matrix satisfies
\[
\mu(Q)<H_*.
\]
\end{lemma}

\begin{proof}
Because $Q=A^T$, the squared norm of a row of $Q$ is the squared norm of the corresponding column of $A$.  We first prove the row-leverage bound
\[
\max_i\|{Q(i,:)}\|_2^2=\alpha^2.
\]

For the first $m$ selected columns, direct summation gives
\[
\|{a_q}\|_2^2
=d_q^2+\beta^2s^2\sum_{t=1}^{q-1}d_t^2
=\alpha^2\left(\beta^2+(1-\beta^2)\epsilon^{2(q-1)}\right)
\le \alpha^2,
\]
with equality at $q=1$.  For the final selected column,
\[
\|{(z^T,\eta_p)}\|_2^2
=\|{z}\|_2^2+\eta_p^2
=\gamma_0^2(\alpha^2-r)+r
<\alpha^2,
\]
because $0<\gamma_0<1$ and $r<\alpha^2$.

For the two columns in $X$,
\[
\|{(\theta_i z^T,-b_i\eta_p)}\|_2^2
=\theta_i^2\|{z}\|_2^2+b_i^2\eta_p^2
<\|{z}\|_2^2+\eta_p^2
<\alpha^2,
\]
since $0<\theta_i,b_i<1$.

For a pure frame column in $F$,
\[
\|{f_h}\|_2^2=W_h<d_m^2\le d_1^2=\alpha^2.
\]
Finally, a residual-only column in $G$ has squared norm
\[
e_j<r/2<\alpha^2.
\]
Thus the maximum row leverage is $\alpha^2$.

Therefore
\[
\mu(Q)=\frac{n}{k}\alpha^2=\frac{n}{k^3}.
\]
Using the count
\[
n=m+3+N_{\mathrm{frame}}+N_{\mathrm{res}},
\]
together with
\[
N_{\mathrm{frame}}\le \frac{m}{d_m^2}+2m,
\qquad
N_{\mathrm{res}}\le \frac{2}{\tau_0d_m^2}+2,
\]
we estimate the two filler contributions separately.  Since
\[
d_m^2=\alpha^2\epsilon^{2m-2}
=\frac1{k^2}\exp\left(-\frac{2\eta(m-1)}{k}\right),
\]
we have
\[
\frac{m}{k^3d_m^2}
=\frac{k-1}{k}\exp\left(\frac{2\eta(k-2)}{k}\right)
\longrightarrow e^{2\eta}.
\]
The remaining terms contribute only $O(k^{-1})$ after division by $k^3$.  Hence
\[
\mu(Q)\le e^{2\eta}+o(1).
\]
Because $e^{2\eta}<H_*$, taking $k$ sufficiently large gives $\mu(Q)<H_*$. 
\end{proof}

\begin{lemma}[Selected-block pressure]
\label{lem:B-block-pressure}
Let
\[
Z=\|{R_\beta^{-1}z}\|_2^2.
\]
Then $M(Q)$ is nonsingular and
\[
\gamma(Q,I_{\mathrm{piv}})^2\ge \frac{1+Z}{r}\ge \frac{Z}{r}.
\]
Moreover, $Z$ eventually dominates every fixed polynomial in $k$.
\end{lemma}

\begin{proof}
The matrix $R_+$ is upper triangular with positive diagonal entries $d_1,\ldots,d_m,\eta_p$, so it is nonsingular.

Let $x=R_\beta^{-1}z$ and set
\[
v=\begin{pmatrix}-x\\1\end{pmatrix}.
\]
Then
\[
R_+v=\begin{pmatrix}0\\\eta_p\end{pmatrix},
\qquad
\|{v}\|_2^2=1+Z.
\]
Thus
\[
\sigma_{\min}(R_+)^2\le \frac{\eta_p^2}{1+Z}=\frac{r}{1+Z},
\]
which gives the inverse-norm lower bound.

It remains to see why $Z$ is large.  Write $x_t=\sigma_ty_t$.  Back-substitution gives
\[
y_m=\gamma_0(1-\tau_0)^{1/2},
\qquad
 y_t=s\left(\gamma_0+\beta\sum_{q=t+1}^m y_q\right),\quad t<m.
\]
Hence
\[
y_t=\gamma_0s(1+\beta s)^{m-t-1}\left(1+\beta(1-\tau_0)^{1/2}\right),
\]
and consequently
\[
Z\ge \gamma_0^2s^2(1+\beta s)^{2(m-2)}.
\]
Since $m=k-1$ and there are constants $0<c<C<\infty$ such that
\[
ck^{-1/2}\le s\le Ck^{-1/2},
\]
this lower bound grows faster than any fixed power of $k$.
\end{proof}

\begin{lemma}[Source-scale denominator]
\label{lem:B-denominator}
The row count satisfies
\[
r(n-k+1)\le C_{\mathrm{den}}\,k,
\qquad
C_{\mathrm{den}}=3.
\]
\end{lemma}

\begin{proof}
Since $k=m+1$ and $n=m+3+N_{\mathrm{frame}}+N_{\mathrm{res}}$,
\[
n-k+1=3+N_{\mathrm{frame}}+N_{\mathrm{res}}.
\]
The frame part gives
\[
rN_{\mathrm{frame}}
\le \tau_0d_m^2\left(\frac{m}{d_m^2}+2m\right)
=\tau_0m+2\tau_0d_m^2m.
\]
The residual-only part gives
\[
rN_{\mathrm{res}}
\le 2+2r.
\]
Therefore
\[
r(n-k+1)
\le 3r+\tau_0m+2\tau_0d_m^2m+2+2r.
\]
Because $d_m^2=O(k^{-2})$ and $r=\tau_0d_m^2$, the finite remainder $5r+2\tau_0d_m^2m<1$ for $k$ large. Hence
\[
r(n-k+1)
\le 3+\tau_0m
\le 3k
\]
for all sufficiently large $k$.
\end{proof}

\subsection{Proof of the obstruction theorem}
\label{sec:B-assembly}

We now assemble the preceding lemmas.

\begin{proof}[Proof of \Cref{thm:qrcp-source-scale-obstruction}]
Choose $k=m+1\ge K$ large enough so that the following hold:
\begin{enumerate}[label=(\roman*)]
  \item the coherence estimate in \Cref{lem:B-coherence} gives $\mu(Q)<H_*$;
  \item the denominator estimate in \Cref{lem:B-denominator} holds;
  \item the triangular growth satisfies
  \[
  Z>C_{\mathrm{den}} B^2k^{2e+2}.
  \]
\end{enumerate}
The last condition is possible by \Cref{lem:B-block-pressure}, since $Z$ dominates every fixed polynomial in $k$.

By \Cref{lem:B-parseval}, the matrix $Q=A^T$ has orthonormal columns.  By \Cref{lem:B-pivot-order}, exact QRCP selects
\[
I_{\mathrm{piv}}=(1,2,\ldots,m,m+1).
\]
By \Cref{lem:B-block-pressure}, $M(Q)$ is nonsingular and
\[
\gamma(Q,I_{\mathrm{piv}})^2\ge \frac{Z}{r}.
\]
Therefore
\[
\frac{\gamma(Q,I_{\mathrm{piv}})^2}{k(n-k+1)}
\ge
\frac{Z}{rk(n-k+1)}.
\]
Using \Cref{lem:B-denominator},
\[
r(n-k+1)\le C_{\mathrm{den}} k,
\]
we obtain
\[
\frac{\gamma(Q,I_{\mathrm{piv}})^2}{k(n-k+1)}
\ge
\frac{Z}{C_{\mathrm{den}} k^2}
>
B^2k^{2e}.
\]
Taking square roots gives
\[
\frac{\gamma(Q,I_{\mathrm{piv}})}{\sqrt{k(n-k+1)}}>Bk^e.
\]
This proves the theorem.
\end{proof}

\end{document}